\documentclass[12pt]{article}

\usepackage{amsmath}
\usepackage{amssymb}
\usepackage{amsthm}
\usepackage{tikz}
\usepackage{hyperref}
\usepackage{vmargin}
\usepackage{times}
\usepackage{stmaryrd}
\usepackage{comment}

\setpapersize{USletter}
\setmarginsrb{.9in}{1in}{.9in}{1in}{0cm}{0cm}{.9in}{.2in}

\newcounter{global}
\theoremstyle{definition}
\newtheorem{definition}[global]{Definition}
\theoremstyle{plain}
\newtheorem{theorem}[global]{Theorem}

\newtheorem{lemma}[global]{Lemma}
\newtheorem{corollary}[global]{Corollary}
\newtheoremstyle{note}{}{}{}{}{\itshape}{.}{.5em}{}
\theoremstyle{note}
\newtheorem{remark}{Remark}%
\newtheorem{example}{Example}%
\renewcommand\section{%
  \@startsection {section}{1}{\z@}%
  {-1.5ex \@plus -1ex \@minus -.2ex}%
  {0.3ex \@plus.2ex}%
  {\normalfont\large\bfseries}}
\renewcommand\subsection{%
  \@startsection{subsection}{2}{\z@}%
  {-1.5ex\@plus -1ex \@minus -.2ex}%
  {0.3ex \@plus .2ex}%
  {\normalfont\large\bfseries}}

\def\monoton{isoton} 

\begin{document}

\title{\vspace*{-2.5em}\Large Parameterizing the semantics of fuzzy attribute implications by systems of \monoton e Galois connections}

\date{\normalsize%
  Dept. Computer Science, Palacky University, Olomouc}

\author{\large Vilem Vychodil\footnote{%
e-mail: \texttt{vychodil@binghamton.edu},
phone: +420 585 634 705,
fax: +420 585 411 643}}

\maketitle

\begin{abstract}
  \noindent%
  We study the semantics of fuzzy if-then rules called fuzzy attribute
  implications parameterized by systems of \monoton e Galois connections.
  The rules express dependencies between fuzzy attributes in 
  object-attribute incidence data. The proposed parameterizations
  are general and include as special cases the parameterizations by
  linguistic hedges used in earlier approaches. We formalize the
  general parameterizations, propose bivalent and graded notions of semantic
  entailment of fuzzy attribute implications, show their characterization
  in terms of least models and complete axiomatization, and provide
  characterization of bases of fuzzy attribute implications derived from
  data.

  \medskip\noindent%
\end{abstract}

\def\@leftpar{(}
\def\bigmul#1{\ensuremath{\mathop{\mul\bigl(#1\bigr)}}}
\def\bigshf#1{\ensuremath{\mathop{\shf\bigl(#1\bigr)}}}
\def\Mul{\ensuremath{\mathcal{F}}}
\def\Shf{\ensuremath{\mathcal{G}}}
\def\Sys{\ensuremath{\mathcal{O}}}
\def\I{\ensuremath{\Rightarrow}}
\def\nmodels{\ensuremath{\not\models}}
\def\proves{\ensuremath{\vdash}}
\def\nproves{\ensuremath{\nvdash}}
\def\SD{\ensuremath{\mathrm{S}}}
\def\itm#1{{\rm(\textit{\romannumeral#1})}}
\def\up{\ensuremath{{\shortuparrow_I}}}
\def\dn{\ensuremath{{\shortdownarrow_I}}}
\def\rotat{\ensuremath{\circlearrowleft}}

\makeatletter

\def\@leftpar{(}
\def\@rightpar{)}
\def\bigarg{\def\@leftpar{\bigl(}\def\@rightpar{\bigr)}}
\def\mul{\@ifnextchar[{\@withmul}{\@withoutmul}}
\def\@withmul[#1]{\ensuremath{\boldsymbol{f}_{\!#1}\@ifnextchar\bgroup{\@witharg}{\relax}}}
\def\@withoutmul{\ensuremath{\boldsymbol{f}\@ifnextchar\bgroup{\@witharg}{\relax}}}
\def\shf{\@ifnextchar[{\@withshf}{\@withoutshf}}
\def\@withshf[#1]{\ensuremath{\boldsymbol{g}_{#1}\@ifnextchar\bgroup{\@witharg}{\relax}}}
\def\@withoutshf{\ensuremath{\boldsymbol{g}\@ifnextchar\bgroup{\@witharg}{\relax}}}
\def\one{\@ifnextchar[{\@withone}{\@withoutone}}
\def\@withone[#1]{\ensuremath{\boldsymbol{1}_{#1}\@ifnextchar\bgroup{\@witharg}{\relax}}}
\def\@withoutone{\ensuremath{\boldsymbol{1}\@ifnextchar\bgroup{\@witharg}{\relax}}}
\def\C{\@ifnextchar[{\@withC}{\@withoutC}}
\def\@withC[#1]{\ensuremath{\boldsymbol{c}_{#1}\@ifnextchar\bgroup{\@witharg}{\relax}}}
\def\@withoutC{\ensuremath{\boldsymbol{c}\@ifnextchar\bgroup{\@witharg}{\relax}}}
\def\@witharg#1{\@leftpar #1\@rightpar}

\renewenvironment{thebibliography}[1]
{\section*{\refname}%
  \@mkboth{\MakeUppercase\refname}{\MakeUppercase\refname}%
  \list{\@biblabel{\@arabic\c@enumiv}}%
  {\settowidth\labelwidth{\@biblabel{#1}}%
    \parsep=-2pt
    \leftmargin\labelwidth
    \advance\leftmargin\labelsep
    \@openbib@code
    \usecounter{enumiv}%
    \let\p@enumiv\@empty
    \renewcommand\theenumiv{\@arabic\c@enumiv}}%
  \sloppy
  \clubpenalty4000
  \@clubpenalty \clubpenalty
  \widowpenalty4000%
  \sfcode`\.\@m}
{\def\@noitemerr
  {\@latex@warning{Empty `thebibliography' environment}}%
  \endlist}

\abovedisplayshortskip=5pt
\abovedisplayskip=5pt
\belowdisplayshortskip=5pt
\belowdisplayskip=5pt

\makeatother

\section{Introduction}\label{sec:intro}
Fuzzy if-then rules play a central role in many diverse applications of fuzzy
logic ranging from fuzzy controllers~\cite{Haj:MFL,KlYu} to
data analysis~\cite{BeVy:ADfDwG,DuHuPR:Asaafar} and their applications.
In fuzzy controllers, graded if-then rules which involve linguistic variables
constitute the core of the underlying approximate inference systems.
It is often the case that knowledge bases consisting of if-then rules used
for the approximate inference are prescribed by experts based on their knowledge
of a particular problem domain where the need for automated control arises.
The success of fuzzy controllers is often attributed to the easiness
for non-technically oriented experts to formulate control rules as simple
fuzzy if-then rules~\cite{Za:Itnfl}. From the viewpoint of data analysis,
various types of if-then rules are used to describe dependencies between attribute
values in data collections and, in contrast to their role in controllers, the rules
are often inferred from object-attribute incidence data by specialized algorithms.
We are interested primarily in the data-analytical role of the rules.

In our paper, we consider rules which are syntactically similar to rules which
have been studied earlier in formal concept analysis~\cite{GaWi:FCA}
of data with graded attributes~\cite{BeVy:ICFCA}
as \emph{fuzzy attribute implications}
and similarity-based relational database systems~\cite{BeVy:Codd} as
\emph{similarity-based functional dependencies.}
The principal difference compared to the earlier
approaches is how we \emph{interpret} the rules.
Namely, the present paper shows a general
approach to define semantics of such rules which encompasses the earlier
approaches and, in addition,
allows to consider new types of semantics which have not been
captured by any previous approaches. Even if our approach attains a high level
of generality, we show that most relevant laws regarding the if-then rules
are preserved in the general setting. The rules we consider can be described
as formalizations of data dependencies saying
\begin{align}
  \text{``if } A \text{ (is contained in $M$)},
  \text{ then } B \text{ (is contained in $M$),''}
  \label{eqn:rule}
\end{align}
where $A,B$ and $M$ are fuzzy sets in a universe $Y$. The elements of $Y$ are
called \emph{attributes} and we consider them as symbolic names. The fuzzy sets
$A$ and $B$ in~\eqref{eqn:rule} are called the \emph{antecedent} and
\emph{consequent} of the rule, respectively.
For given $A$ and $B$, the rule is abbreviated by $A \I B$ and can be
seen as a formula in the narrow sense. Strictly speaking, the fuzzy set $M$
in~\eqref{eqn:rule} is not a part of the formula---it represents a semantic
component using which we evaluate the formula $A \I B$.

Obviously, by stating that the formulas under consideration (and their informal
interpretation) can be understood as expressions~\eqref{eqn:rule} we do not define
their formal semantics. As a matter of fact, formal semantics of \eqref{eqn:rule}
may be introduced in many ways and depend on factors like the choice of
the interpretation of the graded material implication (the if-then connective)
and the notion of containment. Also, we may introduce
a \emph{bivalent notion of satisfaction} of the rule (given $M$), i.e.,
the rule either is satisfied or not satisfied (given $M$),
or a \emph{graded notion of satisfaction} expressed by
degrees to which the rule is satisfied (given $M$). Moreover,
the satisfaction of~\eqref{eqn:rule} may be defined in terms of
a \emph{partiality of truth} or a \emph{partiality of confidence}---these
two entirely different concepts should not be mixed or confused
(cf. ``the frequentist's temptation'' in \cite{Haj:MFL,HaPa:Adfl}).
All these factors, and possibly more, may be viewed as
\emph{parameterizations of semantics} of rules like~\eqref{eqn:rule}.

In our paper, we deal with if-then rules like~\eqref{eqn:rule} with general
semantics admitting a graded notion of satisfaction and partiality of truth
(i.e., we consider purely truth-functional approach, cf. also the related
notion of a veristic constraint in~\cite{Za:Ttfig}).
We assume that the degrees of truth come from complete residuated
lattices~\cite{GaJiKoOn:RL,OrRa:Drlta,DiWa} which include widely-used structures,
including linear residuated lattices defined on the real unit-interval by
left-continuous triangular norms or their finite
substructures~\cite{DeMe:TNPL,KMP:TN}. From the perspective of general 
parameterizations, the choice of a complete residuated lattice represents
a choice of one particular parameter of the semantics of~\eqref{eqn:rule}.
Namely, a chosen complete residuated lattice determines the set of degrees
used to express partial truth and truth functions of logical connectives,
most notably the truth function of ``fuzzy implication'' which serves as the
interpretation of the ``if\,\ldots\ then\,\ldots'' connective in~\eqref{eqn:rule}
and, together with general infima, defines the notion of
a \emph{graded containment}~\cite{Gog:Lic}. The first systematic study of the role
of fuzzy if-then rules in the analysis of fuzzy object-attribute data with this
particular type of parameterization goes back to~\cite{Po:FB}.
Later, the approach was extended and
substantially developed in~\cite{BeVy:ICFCA,BeVy:ADfDwG} by
considering \emph{linguistic hedges} \cite{Za:Afstilh} as
additional parameters of the semantics.

The introduction of linguistic hedges as additional parameters brought
several benefits. The hedges allow to put additional emphasis on the
antecedent of~\eqref{eqn:rule}. In practice, this means that by a choice of
a hedge, we can consider rules with different types of containment, e.g.,
we may require that ``$A$ is \emph{almost fully} contained in $M$''.
The approach via hedges allows us to handle the cases of graded
and crisp containment by a single theory. This
aspect is important since some desirable properties of the rules
(like the uniqueness of bases given by pseudo-intents~\cite{BeVy:ADfDwG})
hold only if one considers a particular hedge. As a result, the approach via
hedges simplifies the analysis of properties of the rules and brings
a broader perspective. Let us note that parameterizations by hedges are not
limited only to if-then rules. In \cite{BeVy:Rsclh} which was later
extended in~\cite{BeVy:Fcalh}, we have shown that various approaches
to constrained fuzzy concept lattices \cite{BeSkZa:Cgfc,Krs:Cbefc,YaJa:Dkfcl}
can be seen as approaches to reducing the size of concept lattices by hedges.

In this paper, we focus on parameterizations which are considerably more
general than the linguistic hedges used so far. We show that reasonable
parameterizations may be introduced by considering systems of \monoton e
Galois connections on fuzzy sets. We prove that most properties which are
known for the if-then rules parameterized by hedges~\cite{BeVy:ADfDwG}
are preserved in the general setting. In addition, we show several
non-trivial parameterizations which can be described by systems of \monoton e
Galois connections and cannot be expressed by hedges. The generality of
our approach brings more versatility into the applications of the if-then
rules in data analysis. Indeed, data analysis is inherently an
\emph{interactive process} where experts tune parameters of algorithms in order
to infer information from data in a desirable form. More often than not,
a reasonable output is derived only after several iterations during which
different parameters are used. The present paper offers a formalism which
allows experts to specify parameterizations of if-then rules from a rich family
of parameterizations which supports this interactive process.
For instance, in case of the inference of if-then rules from data,
by a choice of different parameterizations, one influences the number and
meaning of the rules which are extracted from data.

The following main results are shown in our paper: We formalize the
parameterizations by systems of \monoton e Galois connections and show
that if-then rules with this type of parameterization have two notions
of entailment: a \emph{semantic entailment} based on validation in models
and a \emph{syntactic entailment} based on provability using a system of
inference rules. We prove a completeness theorem showing that both types
of entailment coincide. Moreover, we introduce the entailments as crisp
notions as well as graded notions and prove that degrees of entailment
are expressible by the crisp entailment. In addition, we characterize
the degrees of entailment as degrees of containment in least models.
By all these observations we demonstrate that there is a reasonable
logic behind the if-then rules with the considered general parameterizations.
Further results are directly connected to issues of describing dependencies
in given data. We introduce \emph{bases} of if-then rules which represent
non-redundant sets of if-then rules which convey the information about
all if-then rules valid in given data and provide a characterization
of bases using operators on fuzzy sets induced by data. Let us note
that the results we obtain are interesting not only for the graded rules
but also for the classic if-then rules which can be seen as a particular
case of the graded ones when the structure of degrees is the two-element
Boolean algebra. In this setting, the entailment of the graded rules is
equivalent to the entailment of attribute implications~\cite{GaWi:FCA}
and functional dependencies~\cite{Mai:TRD}. Even in this borderline case,
the parameterizations by systems of \monoton e Galois connections brings new
types of semantics of the (classic) if-then rules. This is in contrast with
the earlier approaches by hedges which yield no non-trivial
parameterization in the crisp setting.

Our paper is organized as follows. In Section~\ref{sec:prel}, we recall
preliminary notions from structures of degrees and outline the existing
approaches to parameterizations of if-then rules by hedges.
In Section~\ref{sec:param}, we introduce the general parameterizations
and provide characterization of semantic entailment of the rules in
terms of least models. In Section~\ref{sec:ddd}, we present a description
of non-redundant sets of if-then rules inferred from data. In addition,
in Section~\ref{sec:compl}, we discuss the axiomatization of the semantic
entailment and the relationship of graded vs. crisp notions of entailment.
Finally, Section~\ref{sec:ilex} shows illustrative examples of
parameterizations and their influence on the rules inferred from data.

\section{Preliminaries}\label{sec:prel}
This section contains preliminaries from general structures of truth degrees
we use in our approach, fuzzy attribute implications parameterized by
linguistic hedges which are the starting point of our generalized view
of semantic parameterizations, and closure structures which play a key role
in the semantic parameterizations.

\subsection{Structures of Truth Degrees}
We use general structures of truth degrees which include the most widely-used
structures of degrees in fuzzy logics based on left-continuous triangular
norms~\cite{KMP:TN}. Instead of focusing solely on structures defined
on the real unit interval, we consider general complete lattices, optionally
equipped with additional operations, as the basic structures.
Recall that an ordered set $\mathbf{L} = \langle L,\leq\rangle$
is called a \emph{complete lattice} whenever $\leq$ is a partial order such
that any $K \subseteq L$ has its greatest lower bound
(an \emph{infimum}) in $\mathbf{L}$
denoted $\bigwedge K$ and its least upper bound (a \emph{supremum})
in $\mathbf{L}$
denoted $\bigvee K$. The \emph{least} and the \emph{greatest} elements
in $\mathbf{L}$, which always exist in a complete lattice,
are denoted by $0$ and $1$,
respectively. Alternatively, complete lattices are considered as algebras
$\mathbf{L} = \langle L,\wedge,\vee,0,1\rangle$ where $a \leq b$ if{}f
$a \wedge b = a$ (or, equivalently, $a \vee b = b$), see~\cite{Bir:LT}.

In the paper, we consider additional operations on complete
lattices---multiplications which generalize left-continuous triangular norms
and serve as (truth functions of) fuzzy conjunctions and their residua
which serve as (truth functions of) fuzzy implications.
Let $\mathbf{L} = \langle L,\wedge,\vee,0,1\rangle$ be a complete lattice.
A binary operation $\otimes$ in $\mathbf{L}$ is called
a \emph{multiplication} if it is commutative, associative,
$1$ (the greatest element in $\mathbf{L}$) is its neutral
element (i.e., $a \otimes 1 = a$), and it is distributive
over general suprema~\cite{Gog:LFS}, i.e.,
\begin{align}
  \textstyle\bigvee_{\!i \in I}(a \otimes b_i) &=
  a \otimes \textstyle\bigvee_{\!i \in I}b_i
  \label{eqn:obigvee}
\end{align}
holds for all $a \in L$ and $b_i \in L$ ($i \in I$). For such $\otimes$,
we can consider a binary operation $\rightarrow$
(a \emph{residuum}~\cite{Gog:Lic,KMP:TN,DiWa})
which is given by
\begin{align}
  a \rightarrow b &= \textstyle\bigvee\{c \in L;\, a \otimes c \leq b\}.
\end{align}
It is well known that $\otimes$ and $\rightarrow$ satisfy the following
condition called the \emph{adjointness property}:
\begin{align}
  a \otimes b \leq c \quad \text{if{}f} \quad a \leq b \rightarrow c
  \label{eqn:adj}
\end{align}
for all $a,b,c \in L$. In fact, it can be shown that
under the assumption of $\otimes$ being commutative, associative, and
neutral with respect to $1$, postulating~\eqref{eqn:obigvee} is
equivalent to requiring the existence of $\rightarrow$ satisfying
\eqref{eqn:adj}, cf.~\cite{Bel:FRS,GaJiKoOn:RL}. Altogether,
$\mathbf{L} = \langle L,\wedge,\vee,\otimes,\rightarrow,0,1\rangle$
is called a \emph{complete residuated lattice.} Let us note that
residuated lattices were proposed in~\cite{Di38,DiWa} and their
importance to fuzzy logic and the theory of fuzzy sets was
discovered by J.\,A.~Goguen~\cite{Gog:LFS,Gog:Lic}. Subclasses
of residuated lattices are widely used in applications and constitute
a basis for investigation of mathematical fuzzy
logics~\cite{EsGo:MTL,Ger:FL,Got:Mfl,Haj:MFL}, see also
monographs~\cite{CiHaNo1,CiHaNo2} devoted to recent results.

The class of complete residuated lattices is rich and it includes infinite
as well as finite structures. Frequently used infinite structures include
linearly ordered complete residuated lattices defined on the real unit
interval with $\wedge$ and $\vee$ being minimum and maximum,
$\otimes$ being a left-continuous (or a continuous) triangular norm
with the corresponding~$\rightarrow$, see~\cite{KMP:TN}. Three most
important continuous triangular norms and their residua are
the \L ukasiewicz, G\"odel (or minimum), and Goguen (or product)
adjoint operations:
\begin{align}
  a \otimes b &= \max(a + b - 1, 0), 
  &
  a \otimes b &= \min(a,b),
  &
  a \otimes b &= a \cdot b,
  \label{eqn:LukGodGog_mul}
  \\
  a \rightarrow b &= \min(1 - a + b, 1),
  &
  a \rightarrow b &=
  \left\{
    \begin{array}{@{\,}l@{\quad}l@{}}
      1, & \text{if } a \leq b, \\
      b, & \text{otherwise,}
    \end{array}
  \right.
  &
  a \rightarrow b &=
  \left\{
    \begin{array}{@{\,}l@{\quad}l@{}}
      1, & \text{if } a \leq b, \\
      \frac{b}{a}, & \text{otherwise.}
    \end{array}
  \right.
  \label{eqn:LukGodGog_res}
\end{align}

In the paper we utilize addtional operations on $L$ called idempotent
truth-stressing linguistic hedges~\cite{Za:Afstilh,Za:lv1,Za:lv2,Za:lv3}
which have been used to parameterize the semantics of fuzzy attribute
implications in earlier approaches~\cite{BeVy:Fcalh,BeVy:ADfDwG}.
An idempotent truth-stressing linguistic hedge (shortly, a hedge)
on $\mathbf{L}$ is a map ${}^*\!: L \to L$ such that 
\begin{align}
  1^{\ast} &= 1,
  \label{ts:1}
  \\
  a^{\ast} &\leq a,
  \label{ts:sub}
  \\
  (a \rightarrow b)^{\ast} &\leq a^{\ast} \rightarrow b^{\ast},
  \label{ts:mon}
  \\
  a^{\ast\ast} &= a^{\ast}
  \label{ts:idm}
\end{align}
for all $a,b \in L$. Operations on $L$ satisfying~\eqref{ts:1}--\eqref{ts:idm}
may be seen as truth functions of logical connectives ``very true''.
Technically, the hedges we consider are generalizations of Baaz's $\Delta$
operation~\cite{Baaz,Haj:MFL} and they have been studied in fuzzy logics
in the narrow sense~\cite{Got:Mfl} by H\'ajek in~\cite{Haj:Ovt},
see also~\cite{EsGoNo:Hedges} for a recent general approach. Every $\mathbf{L}$
admits two borderline hedges: (i) the identity (i.e., $a^* = a$ for
any $a \in L$), and (ii) the so-called \emph{globalization}~\cite{TaTi:Gist}:
\begin{align}
  a^{\ast} = \left\{
    \begin{array}{@{\,}l@{\quad}l@{}}
      1, & \text{if } a = 1, \\
      0, & \text{otherwise.}
    \end{array}
  \right.
  \label{eqn:glob}
\end{align}

Using complete residuated lattices as the structures of truth degrees,
we consider the usual notions of $\mathbf{L}$-sets (fuzzy sets) and
$\mathbf{L}$-relations (fuzzy relations), see~\cite{Bel:FRS,Gog:LFS,KlYu}.
That is, for a non-empty set $Y$ (call it a universe), we may consider
a map $A\!: Y \to L$, assigning to each $y \in Y$ a degree $A(y) \in L$.
Such a map is called an \emph{$\mathbf{L}$-set} in $Y$. The system of
all $\mathbf{L}$-sets in $Y$ is denoted by~$L^Y$.
For $Y$ and $c \in L$, we consider a \emph{constant $\mathbf{L}$-set}
$c_Y \in L^Y$ defined by
\begin{align}
  c_Y(y) = c
  \label{eqn:cY}
\end{align}
for all $y \in Y$. In particular, $0_Y$ satisfies $0_Y(y) = 0$ ($y \in Y$)
and is called the \emph{empty $\mathbf{L}$-set} in $Y$.

If $Y$ is a Cartesian product of sets (say $Y_1,\dots,Y_n$),
we may call an $\mathbf{L}$-set $A$ in $Y$
an $\mathbf{L}$-relation (between $Y_1,\dots,Y_n$).
In particular, a binary $\mathbf{L}$-relation $R$
between $X$ and $Y$ is a map $R\!: X \times Y \to L$ assigning to
each $x \in X$ and each $y \in Y$ the degree $R(x,y) \in L$ to which
$x$ and $y$ are related by $R$.
If $Y = \{y_1,\dots,y_n\}$, we use the usual convention for writing
$\mathbf{L}$-sets $A \in L^Y$ as
$\{{}^{a_1\!}/y_1,\dots,{}^{a_n\!}/y_n\}$ meaning
that $A(y_i) = a_i$ for all $i=1,\dots,n$. In addition,
we omit ${}^{a_i\!}/y_i$ if $a_i = 0$ and write $y_i$ if $a_i = 1$.

For $\mathbf{L}$-sets (and $\mathbf{L}$-relations), we may consider two
basic subsethood relations. First, for $A,B \in L^Y$, we write
\begin{align}
  A \subseteq B
  \text{ whenever }
  A(y) \leq B(y) \text{ holds for each } y \in Y
  \label{eqn:crisp_sub}
\end{align}
and say that $A$ is (\emph{fully}) \emph{contained} in $B$.
Second, for $A,B \in Y$, we put
\begin{align}
  \SD(A,B) &= \textstyle\bigwedge_{y \in Y}\bigl(A(y) \rightarrow B(y)\bigr)
  \label{eqn:SD}
\end{align}
and call $\SD(A,B)$ the \emph{subsethood degree} of $A$ in $B$, i.e.,
$\SD(A,B) \in L$ is a degree to which $A$ is included in
$B$~\cite{Gog:Lic}.
We have $\SD(A,B) = 1$ if{}f $A \subseteq B$, see~\cite[Theorem 3.12]{Bel:FRS}.
We further utilize operations with $\mathbf{L}$-sets
(and $\mathbf{L}$-relations) defined componentwise using operations
in $\mathbf{L}$. That is, for $A_i \in L^Y$ ($i \in I$), we put
\begin{align}
  \textstyle\bigl(\bigcap_{i \in I}A_i\bigr)(y) &=
  \textstyle \bigwedge_{i \in I}A_i(y),
  \label{eqn:cap}
  \\
  \textstyle\bigl(\bigcup_{i \in I}A_i\bigr)(y) &=
  \textstyle \bigvee_{\!i \in I}A_i(y),
  \label{eqn:cup}
\end{align}
for all $y \in Y$ and call $\bigcap_{i \in I}A_i$ and $\bigcap_{i \in I}A_i$
the (idempotent) \emph{intersection} and \emph{union} of $A_i$'s,
respectively. For $A,B \in L^Y$, we use the usual infix notation 
$A \cap B$ and $A \cup B$. Analogously, we may define operations componentwise
based on $\otimes$ and $\rightarrow$ in $\mathbf{L}$ as follows:
\begin{align}
  (A \otimes B)(y) &= A(y) \otimes B(y),
  \label{eqn:otimes}
  \\
  (A \rightarrow B)(y) &= A(y) \rightarrow B(y).
  \label{eqn:rightarrow}
\end{align}
Note that $\otimes$ and $\rightarrow$ on the left-hand sides
of~\eqref{eqn:otimes} and \eqref{eqn:rightarrow} denote operations with
$\mathbf{L}$-sets while the symbols of the right-hand sides of the equalities
denote the operations in $\mathbf{L}$. As a particular
cases of~\eqref{eqn:otimes} and~\eqref{eqn:rightarrow} which utilize constant
$\mathbf{L}$-sets, for every $c \in L$ and $A \in L^Y$, we introduce
\begin{align}
  c \otimes A &= c_Y \otimes A,
  \label{eqn:a-mult}
  \\
  c \rightarrow A &= c_Y \rightarrow A,
  \label{eqn:a-shift}
\end{align}
and call $c \otimes A$ and $c \rightarrow A$ the $c$-multiple and the $c$-shift of $A$, respectively. Recall that $c_Y$ which is used here is
defined by~\eqref{eqn:cY}.

\begin{remark}
  Our system does not have a negation as a fundamental operation.
  This is in contrast with approaches which use $1-a$ as the (truth function of)
  negation of $a \in [0,1]$, see~\cite{Za:FS}.
  In our case, $1-a$ does not make sense because
  we work with general structures of truth degrees which may not be defined
  on the real unit interval. More importantly, the negation is not essential
  for the presented results. Nevertheless, reasonable truth functions of
  negations may be defined using $\rightarrow$ in $\mathbf{L}$ as
  $a \rightarrow 0$ for all $a \in L$,
  see~\cite{EsGo:MTL,EsGoHaNa:Rflin,Haj:MFL,Ho:ML} for details.
\end{remark}

In the rest of the paper, $\mathbf{L}$ always denotes a complete (residuated)
lattice. Moreover, we use the fact that for a universe set $Y$, the collection
$L^Y$ of all $\mathbf{L}$-sets in $Y$ together with the full containment
relation $\subseteq$ defined by~\eqref{eqn:crisp_sub} is also a complete
lattice and we denote it by $\langle L,\subseteq\rangle$. In addition,
$\langle L,\subseteq\rangle$ equipped with $\otimes$ and $\rightarrow$
defined componentwise using the operations in $\mathbf{L}$ as
in~\eqref{eqn:otimes} and~\eqref{eqn:rightarrow}
is a complete residuated lattice, cf.~\cite[Theorem 3.6]{Bel:FRS}.

\subsection{Fuzzy Attribute Implications}\label{sec:pfai}
Let $\mathbf{L}$ be a complete residuated lattice and $Y$ be a non-empty
set of symbols called \emph{attributes.} A \emph{fuzzy attribute implication}
(or a \emph{graded} attribute implication, shortly a FAI) in $Y$ is an
expression $A \I B$, where $A,B \in L^Y$. The intended meaning of $A \I B$
is to express data dependency saying that if each attribute $y \in Y$
is present at least to degree $A(y)$, then each $y \in Y$ is present at least
to degree $B(y)$.
Considering $M \in L^Y$ as an $\mathbf{L}$-set representing the presence of
attributes (i.e., $M(y)$ is a degree to which $y \in Y$ is present),
we define the degree $||A \I B||^*_M$ to which $A \I B$ is true
in $M \in L^Y$ by
\begin{align}
  ||A \I B||^*_M &= \SD(A,M)^* \rightarrow \SD(B,M),
  \label{eqn:hedge_fai}
\end{align}
where $\SD$ is the graded subsethood~\eqref{eqn:SD}, $\rightarrow$ is the
residuum (a fuzzy implication) in $\mathbf{L}$, and ${}^*$ is
a truth-stressing hedge which serves as an additional parameter of
the interpretation of $A \I B$.

\begin{remark}\label{rem:semexpl}
  The first approach to fuzzy attribute implications and investigation of
  their role in formal concept analysis of graded incidence data
  goes back to Pollandt~\cite{Po:FB}. The parameterization of FAIs by
  hedges~\eqref{eqn:hedge_fai} was proposed later, see~\cite{BeVy:ADfDwG}
  for a survey of results. Using hedges, we encapsulate different possible
  interpretations of FAIs and can approach them by a single theory.
  The following cases which result by two borderline choices of hedges
  are especially important:
  \begin{enumerate}\parskip=0pt%
  \item
    When ${}^*$ is globalization~\eqref{eqn:glob},
    then $||A \I B||^*_M = 1$ means that
    $A \subseteq M$ implies $B \subseteq M$, where $\subseteq$ is the full
    containment of $\mathbf{L}$-sets defined by \eqref{eqn:crisp_sub}.
    Therefore, $A,B \in L^Y$ can be seen as prescribing threshold for each
    attribute $y \in Y$. If $A \nsubseteq M$, i.e.,
    the attributes (in $M$) are not
    present to the prescribed threshold degrees, we get $\SD(A,M) < 1$ and so
    $\SD(A,M)^* = 0$, meaning that $||A \I B||^*_M = 1$.
    If $A \subseteq M$, then $||A \I B||^*_M = \SD(B,M)$,
    i.e., $A \I B$ is true to the degree to which $B$ is contained in $M$.
  \item
    When ${}^*$ is identity, then $||A \I B||^*_M = 1$ means that
    $\SD(A,M) \rightarrow \SD(B,M) = 1$ which is true if{}f
    $\SD(A,M) \leq \SD(B,M)$. Put in words, $||A \I B||^*_M = 1$ if{}f
    the degree to which the attributes in $B$ are present (in $M$) is
    at least as high as the degree to which the attributes in $A$
    are present (in $M$).
  \end{enumerate}
  Therefore, setting ${}^*$ to globalization and identity represents two
  possible (and both reasonable) interpretations of FAIs. The cases of other
  hedges can be seen as transitions between these two borderline semantics.
  More detailed explanation can be found in~\cite[Remark 3.3]{BeVy:ADfDwG}.
\end{remark}

Let $\Sigma$ be a set of FAIs in $Y$. An $\mathbf{L}$-set $M \in L^Y$ is
called a \emph{model} of $\Sigma$ if $||A \I B||^*_M = 1$ for all
$A \I B \in \Sigma$. Let $\mathrm{Mod}^*(\Sigma)$ denote the set of
all models of $\Sigma$. That is,
\begin{align}
  \mathrm{Mod}^*(\Sigma) &=
  \bigl\{M \in L^Y\!;\,
  ||A \I B||^*_M = 1 \text{ for all } A \I B \in \Sigma
  \bigr\}.
  \label{eqn:Mod*}
\end{align}
Let $A \I B$ be a FAI in $Y$. The
\emph{degree $||A \I B||^*_\Sigma$ to which $A \I B$
  is semantically entailed by $\Sigma$} is defined by
\begin{align}
  ||A \I B||^*_\Sigma &=
  \textstyle\bigwedge_{M \in \mathrm{Mod}^*(\Sigma)}||A \I B||^*_M.
  \label{eqn:sement*}
\end{align}
Put in words, $||A \I B||^*_\Sigma$ is a degree to which $A \I B$
is true in all models of~$\Sigma$, i.e., it is the greatest lower
bound of truth degrees to which $A \I B$ is true in all models.

The semantic entailment has a complete axiomatization~\cite{BeVy:ADfDwG}
which is based on a system of axioms and two inference rules and
resembles the classic Armstrong axiomatic system~\cite{Arm:Dsdbr}
for functional dependencies. Namely, each FAI of the form $A {\cup} B \I A$
($A,B \in L^Y$) is an \emph{axiom}. In addition,
we introduce an inference rule
\begin{align}
  \dfrac{A \I B,\,B \cup C \I D}{A \cup C \I D}
  \quad \text{for all } A,B,C,D \in L^Y,
  \label{r:Cut}
\end{align}
which is called a \emph{cut} (or \emph{pseudo-transitivity},
see~\cite{Hol,Mai:TRD}) and an inference rule
\begin{align}
  \dfrac{A \I B}{c^*{\otimes}A \I c^*{\otimes}B}
  \quad \text{for all } A,B \in L^Y \text{ and } c \in L,
  \label{r:Mul}
\end{align}
which is called a \emph{$c^*$-multiplication.} As it is usual,
the inference rules should be read 
``from $A \I B$ and $B{\cup}C \I D$ infer $A{\cup}C \I D$''
in case of~\eqref{r:Cut} and
``from $A \I B$ infer $c^*{\otimes}A \I c^*{\otimes}B$''
in case of~\eqref{r:Mul}. Note that in \eqref{r:Mul},
$c^*{\otimes}A$ and $c^*{\otimes}B$ are defined by~\eqref{eqn:a-mult}.
A notion of provability by $\Sigma$ is defined
the usual way based on the existence of a finite sequence of formulas
which are either assumptions from $\Sigma$, axioms, or are inferred
from some preceding formulas in the sequence by~\eqref{r:Cut}
or~\eqref{r:Mul}.
Let us stress at this point that analogously as in the case of semantic
entailment, the hedge ${}^*$ serves as a parameter of the provability
because it appears explicitly in~\eqref{r:Mul}. As a consequence,
considering different hedges changes the inference system.

\subsection{Galois Connections and Closure Structures}
Let $\mathbf{L}$ be a complete (residuated) lattice and let
$\langle L^Y, \subseteq\rangle$ be the complete lattice of $\mathbf{L}$-sets
in $Y$. A~pair $\langle \mul,\shf\rangle$ of operators $\mul\!: L^Y \to L^Y$
and $\shf\!: L^Y \to L^Y$ is called
an \emph{\monoton e Galois connection}~\cite{DaPr}
in $\langle L^Y,\subseteq\rangle$ whenever
\begin{align}
  \mul{A} \subseteq B \text{ if{}f } A \subseteq \shf{B}
  \label{eqn:gal}
\end{align}
for all $A,B \in L^Y$; $\mul$ is called the \emph{lower adjoint} of~$\shf$ and,
dually, $\shf$ is called the \emph{upper adjoint} of~$\mul$.
In an \monoton e Galois connection $\langle \mul,\shf\rangle$,
$\mul$ uniquely determines $\shf$ and \emph{vice versa}. In particular,
\begin{align}
  \mul{A} &= \textstyle\bigcap\{B \in L^Y\!;\, A \subseteq \shf{B}\},
  \\
  \shf{B} &= \textstyle\bigcup\{A \in L^Y\!;\, \mul{A} \subseteq B\}.
\end{align}
In our paper, we utilize the following properties which follow
by~\eqref{eqn:gal}. For any $A,B \in L^Y$, $A_i \in L^Y$ ($i \in I$),
and $B_i \in L^Y$ ($i \in I$), the following properties hold:
\begin{align}
  &A \subseteq \shf{\mul{A}},
  \label{eqn:shfmul} \\
  &\mul{\shf{B}} \subseteq B
  \label{eqn:mulshf} \\
  &A \subseteq B \text{ implies } \mul{A} \subseteq \mul{B},
  \label{eqn:mon_mul} \\
  &A \subseteq B \text{ implies } \shf{A} \subseteq \shf{B},
  \label{eqn:mon_shf} \\
  &\bigarg \mul{\textstyle\bigcup_{i \in I}A_i} =
  \textstyle\bigcup_{i \in I}\mul{A_i},
  \label{eqn:mul_distr} \\
  &\bigarg \shf{\textstyle\bigcap_{i \in I}B_i} =
  \textstyle\bigcap_{i \in I}\shf{B_i}.
\end{align}
We also utilize \emph{composition} of \monoton e Galois connections.
That is, for \monoton e Galois connections 
$\langle\mul[1],\shf[1]\rangle$ and $\langle\mul[2],\shf[2]\rangle$
in $\langle L^Y,\subseteq\rangle$, we put
\begin{align}
  \langle \mul[1],\shf[1]\rangle \circ \langle\mul[2],\shf[2]\rangle
  = \langle \mul[1]\mul[2],\shf[2]\shf[1]\rangle,
  \label{eqn:compose}
\end{align}
where $\mul[1]\mul[2]$ is a composed operator such that 
$\mul[1]\mul[2](A) = \mul[1](\mul[2](A))$ for all $A \in L^Y$
and analogously for $\shf[2]\shf[1]$. It is easy to see that the
composition is again an \monoton e Galois connection in
$\langle L^Y,\subseteq\rangle$. Furthermore, we denote by $\one$
the operator in $L^Y$ such that $\one{A} = A$ for any $A \in L^Y$.
Then, $\langle \one,\one\rangle$ is trivially an \monoton e Galois
connection in $\langle L^Y,\subseteq\rangle$. All \monoton e Galois
connections in $\langle L^Y,\subseteq\rangle$ together with $\circ$
defined by~\eqref{eqn:compose} and $\langle\one,\one\rangle$ form
a monoid (i.e., an algebra with associative binary operation $\circ$
with a neutral element $\langle\one,\one\rangle$).

An operator $\C\!: L^Y \to L^Y$ is called an
\emph{$\mathbf{L}^*$-closure operator}~\cite{BeFuVy:Fcots}
in $\langle L^Y, \subseteq\rangle$ whenever
\begin{align}
  A &\subseteq \C{A},
  \label{eqn:*ext} \\
  \SD(A,B)^* &\leq \SD(\C{A},\C{B}),
  \label{eqn:*mon} \\
  \C{\C{A}} &\subseteq \C{A},
  \label{eqn:*idm}
\end{align}
for all $A,B \in L^Y$ (recall that ${}^*$ is a hedge on $\mathbf{L}$).
A system $\mathcal{S} \subseteq L^Y$ is called
an \emph{$\mathbf{L}^*$-closure system} in $\langle L^Y, \subseteq\rangle$
whenever it is closed under arbitrary intersections and
\begin{align}
  M \in \mathcal{S}
  \text{ implies } a^* \rightarrow M \in \mathcal{S},
\end{align}
for each $M \in \mathcal{S}$ and $a \in L$. There is a one-to-one
correspondence between $\mathbf{L}^*$-closure systems and
operators~\cite{BeFuVy:Fcots}. The structures play an important role
in analysis on object-attribute data with fuzzy attributes,
see~\cite{BeVy:Fcalh}.
In our case, it is important to note that $\mathrm{Mod}^*(\Sigma)$
(for any $\Sigma$) is an $\mathbf{L}^*$-closure system and the corresponding
$\mathbf{L}^*$-closure operator maps any $M \in L^Y$ to the least model
of $\Sigma$ containing $M$. In fact, in terms of their expressive power,
the systems of models of FAIs parameterized by hedges are
exactly the $\mathbf{L}^*$-closure systems, cf.~\cite{Vy:Faiep}.

\section{Parameterizations of FAIs}\label{sec:param}
In this section, we first formalize the notion of a parameterization and then
we use it to define the notions of truth, models, and semantic entailment of FAIs.
Note that in the case of the parameterizations by hedges outlined in
Section~\ref{sec:pfai}, the parameterization is given by the chosen hedge.
As we shall see in this section, the essence of the parameterization by hedges
is captured just by considering the fixed points of hedges (here, by
a fixed point of ${}^*$ we mean any $a \in L$ such that $a^{**} = a^*$)
which induce particular \monoton e Galois connection. This motivates us to
consider general parameterizations as systems of \monoton e Galois connections.
The following definition summarizes our requirements on such systems.

\begin{definition}\label{def:Param}
  Let $S$ be a set of \monoton e Galois connections in
  $\langle L^Y,\subseteq\rangle$ which contains $\langle \one,\one\rangle$ and
  is closed under composition. The algebra
  $\mathbf{S} = \langle S,\circ,\langle \one,\one\rangle\rangle$
  is called a \emph{parameterization} of FAIs.
\end{definition}

According to Definition~\ref{def:Param}, the parameterizations are exactly
the submonoids of the monoid of all \monoton e Galois connections in
$\langle L^Y,\subseteq\rangle$. In our paper, we assume this to be the
weakest reasonable condition which is used to define interpretation
of FAIs with reasonably strong properties as we shall see later.
Given $\mathbf{S}$, we introduce notions related to the semantic
entailment as follows.

\begin{definition}\label{def:truth}
  Let $A,B,M \in L^Y$ and let $\mathbf{S}$ be a parameterization of FAIs.
  We say that \emph{$A \I B$ is true in $M$ under $\mathbf{S}$},
  written $M \models^\mathbf{S} A \I B$, whenever
  \begin{align}
    \mul{A} \nsubseteq M \text{ or } \mul{B} \subseteq M
    \label{eqn:Strue}
  \end{align}
  holds for all $\langle \mul,\shf\rangle \in S$.
  Otherwise, we write $M \nmodels^\mathbf{S} A \I B$.
  Let $\Sigma$ be a set of FAIs in $Y$. We say that $M \in L^Y$ is
  an \emph{$\mathbf{S}$-model of $\Sigma$} whenever
  $M \models^\mathbf{S} C \I D$ for all $C \I D \in \Sigma$.
  The set of all $\mathbf{S}$-models of $\Sigma$ is denoted by
  $\mathrm{Mod}^\mathbf{S}(\Sigma)$. Furthermore,
  \emph{$A \I B$ is semantically entailed by $\Sigma$ under~$\mathbf{S}$},
  written $\Sigma \models^\mathbf{S} A \I B$,
  whenever $\mathrm{Mod}^\mathbf{S}(\Sigma) \subseteq
  \mathrm{Mod}^\mathbf{S}(\{A \I B\})$.
\end{definition}

\begin{remark}
  (a)
  Let us note that using the standard relationship between the material
  implication, negation, and disjunction,
  we can restate condition~\eqref{eqn:Strue} as
  \begin{align}
    \text{if } \mul{A} \subseteq M \text{, then } \mul{B} \subseteq M.
    \label{eqn:true_f}
  \end{align}
  Furthermore, taking into account the fact that
  $\langle\mul,\shf\rangle \in S$ is an \monoton e Galois connection,
  using~\eqref{eqn:gal}, the previous condition is equivalent to
  \begin{align}
    \text{if } A \subseteq \shf{M} \text{, then } B \subseteq \shf{M}.
    \label{eqn:true_g}
  \end{align}
  Thus, the notion of $A \I B$ being true in $M$ under $\mathbf{S}$ may
  be equivalently defined using the upper adjoints in $\mathbf{S}$
  instead of the lower adjoints in $\mathbf{S}$.

  (b)
  The notions of an $\mathbf{S}$-model and semantic entailment under
  $\mathbf{S}$ are defined using the notion of truth in a standard way
  but they both depend on chosen $\mathbf{S}$. Note that
  Definition~\ref{def:truth} defines the truth of FAIs and the semantic
  entailment as bivalent notions.
  In Section~\ref{sec:compl}, we show that reasonable graded notions can
  also be introduced and, in addition, they are expressible by the bivalent
  ones. Considering this important fact, we investigate properties of the
  bivalent notions and later show how they can be used to reason about
  their graded counterparts.
\end{remark}

It is important to show that parameterizations by hedges may be viewed as
parameterizations by some systems of \monoton e Galois connections and thus
the proposed notions constitute a proper generalization of the parametrization
by hedges. This is shown in the following example. It also shows examples of
other non-trivial parameterizations which can be handled in our approach.

\begin{example}\label{ex:S}
  (a)
  Let us first consider a parameterization $\mathbf{S}$ such that
  $S = \{\langle\one,\one\rangle\}$. Trivially, $S$ is closed under
  composition and contains $\langle\one,\one\rangle$ so it is indeed
  a parameterization according to Definition~\ref{def:Param}.
  Inspecting~\eqref{eqn:Strue}, $M \models^\mathbf{S} A \I B$ means
  that $A = \one{A} \subseteq M$ implies $B = \one{B} \subseteq M$ which
  holds if{}f $||A \I B||^*_M = 1$ where ${}^*$ is~\eqref{eqn:glob}.
  Therefore, the case of $S = \{\langle\one,\one\rangle\}$ covers the
  semantics of FAIs parameterized by globalization,
  cf. Remark~\ref{rem:semexpl}.

  (b)
  For each $c \in L$ and $A,B \in L^Y$,
  we may consider the following operators:
  \begin{align}
    \mul[c \otimes]{A} &= c \otimes A,
    \label{eqn:MulC}
    \\
    \shf[c \rightarrow]{B} &= c \rightarrow B.
    \label{eqn:ShfC}
  \end{align}
  In fact, $\mul[c \otimes]$ and $\shf[c \rightarrow]$ represent $c$-multiples
  and $c$-shifts of $\mathbf{L}$-sets, see~\eqref{eqn:a-mult}
  and~\eqref{eqn:a-shift}. Clearly, \eqref{eqn:adj} yields that
  $\langle\mul[c \otimes],\shf[c \rightarrow]\rangle$ is an \monoton e Galois
  connection. Now, let ${}^*$ be a general idempotent truth-stressing hedge
  on $\mathbf{L}$ and put
  \begin{align}
    S^* &= \{\langle\mul[c^* \otimes],\shf[c^* \rightarrow]\rangle;\, c \in L\}.
    \label{eqn:S*}
  \end{align}
  For $c = 1$, we get $\langle\one,\one\rangle \in S^*$. In addition, $S^*$
  is closed under compositions because for any $a,b \in L$ and
  $c = a^* \otimes b^*$, we have
  $\mul[a^* \otimes]\mul[b^* \otimes] = \mul[c^* \otimes ]$
  and 
  $\shf[b^* \rightarrow]\shf[a^* \rightarrow] = \shf[c^* \rightarrow ]$.
  This is a direct consequence of the fact that
  $a^* \otimes b^* = (a^* \otimes b^*)^*$ for all $a,b \in L$,
  see~\cite[Lemma 2]{BeVy:MRAP}. Therefore, $\mathbf{S}^*$ may be used
  as a parameterization of the semantics of FAIs.
  Then, $M \models^\mathbf{S} A \I B$
  means that, for any $c \in L$, we have that $c^* {\otimes} A \subseteq M$
  implies $c^* {\otimes} B \subseteq M$.
  Using~\cite[Lemma 3.13]{BeVy:ADfDwG}, the last condition is equivalent
  to stating that $||A \I B||^*_M = 1$. Therefore, parameterizations by hedges
  are indeed a special case of our general approach.

  (c)
  As a particular case of (b), we may consider 
  $S = \{\langle\mul[c \otimes],\shf[c \rightarrow]\rangle;\, c \in L\}$
  which coincides with the parameterization by ${}^*$ considered as the
  identity on $L$ which corresponds to the approach by Pollandt~\cite{Po:FB}.
  Therefore, $M \models^\mathbf{S} A \I B$ if{}f
  $\SD(A,M) \leq \SD(B,M)$, cf. Remark~\ref{rem:semexpl}.

  (d)
  A~more general approach than using a single hedge is to introduce
  a~hedge for any attribute. Concept lattices constrained by hedges in
  this sense are studied in~\cite{BeVy:Fcalh}. In our setting, we may
  consider FAIs with an analogous type of a parameterization.
  In general, we may start by considering
  an $I$-indexed system of $\mathbf{L}$-sets
  $C_i \in L^Y$ ($i \in I$) and put
  \begin{align}
    S' = \{\langle\mul[C_i \otimes],\shf[C_i \rightarrow]\rangle;\, i \in I\},
    \label{eqn:Sprime}
  \end{align}
  where $\mul[C_i \otimes]$ is defined by
  $\mul[C_i \otimes]{A} = C_i \otimes A$ for all $A \in L^Y$ as
  in~\eqref{eqn:otimes}
  and analogously for $\shf[C_i \rightarrow]$ as
  in~\eqref{eqn:rightarrow}. In general, $S'$ is not closed under $\circ$
  nor $\langle\one,\one\rangle$ belongs to $S'$ but we can consider $S$ which
  is generated by $S'$, see~\cite[page 11]{Wec:UAfCS}.
  That is $\mathbf{S} = \langle S,\circ,\langle\one,\one\rangle\rangle$
  is the least monoid which contains $S'$.
  We then have $M \models^\mathbf{S} A \I B$ if{}f for every sequence
  $i_1,\ldots,i_n$ (including $n=0$) such that
  $\mul[C_{i_1} \otimes \cdots \otimes C_{i_n} \otimes]{A} \subseteq M$
  we have 
  $\mul[C_{i_1} \otimes \cdots \otimes C_{i_n} \otimes]{B} \subseteq M$
  considering $\mul[\otimes] = \shf[\rightarrow] = \one$.
  Now, if each attribute has its hedge ${}^{*_y}$ as in~\cite{BeVy:Fcalh}
  and if $F_y \subseteq L$ is the set of all its fixed points,
  then for $I = \prod_{y \in Y} F_y$, we may put $C_i(y) = i(y)$
  for all $i \in I$ and $y \in Y$. In this particular case,
  \eqref{eqn:Sprime} is already closed under $\circ$ and contains
  $\langle\one,\one\rangle$, i.e.,  
  $\mathbf{S}' = \langle S',\circ,\langle\one,\one\rangle\rangle$
  follows Definition~\ref{def:Param}.

  (e)
  Further parameterizations may be obtained analogously as in case of (c)
  using $\oplus$ (a generalization of a triangular co-norm~\cite{KMP:TN})
  which is adjoint to $\ominus$.
  Namely, we may consider binary operations $\oplus$
  (called an \emph{addition}) and $\ominus$ (called a \emph{difference})
  in $\mathbf{L}$ such that $\oplus$ is commutative, associative, and has $0$
  as its neutral element, and
  \begin{align}
    a \ominus b \leq c \quad \text{if{}f} \quad a \leq b \oplus c
    \label{eqn:dualadj}
  \end{align}
  for all $a,b,c \in L$, see~\cite{GaJiKoOn:RL,OrRa:Drlta}. For any $C \in L^Y$,
  we put
  \begin{align}
    (\mul[\ominus C]{A})(y) &= A(y) \ominus C(y),
    \label{eqn:mul_ominus}
    \\
    (\shf[C \oplus]{B})(y) &= C(y) \oplus B(y),
    \label{eqn:shf_oplus}
  \end{align}
  for all $A,B \in L^Y$ and $y \in Y$. Clearly, \eqref{eqn:dualadj} yields
  that $\langle\mul[\ominus C],\shf[C \oplus]\rangle$ is an \monoton e Galois
  connection and we may consider parameterizations generated by collections
  of such connections as in (d). Note that for $\mathbf{L}$ defined on the
  real unit interval by a left-continuous t-norm $\otimes$,
  we can consider $\oplus$ defined by 
  $a \oplus b = 1 - ((1 - a) \otimes (1 - b))$ for all $a,b \in [0,1]$ and
  the adjoint operation $\ominus$ is then
  $a \ominus b = 1 - ((1 - b) \rightarrow (1 - a))$ for all $a,b \in [0,1]$.
  For illustration, Figure~\ref{fig:mulshf} depicts $\mul[c\otimes]$,
  $\shf[c\rightarrow]$, $\mul[\ominus c]$, and $\shf[c \oplus]$ in case of
  $c = \frac{1}{3}$ for the three most important pairs
  $\langle \otimes,\rightarrow\rangle$ of adjoint operations and the
  corresponding duals $\langle \oplus,\ominus\rangle$.
  The bold contour shows the result
  of these operators applied to the bell-shaped $\mathbf{L}$-set drawn by
  the dotted line. The dashed horizontal line marks the degree
  $c = \frac{1}{3}$.

  \begin{figure}[t]
    \centering%
    \def\move#1{\lower-2.5em\hbox{#1}}%
    \begin{tabular}{@{}c@{~}c@{\quad}c@{\quad}c@{}}
      \move{$\mul[c\otimes]$} &
      \includegraphics{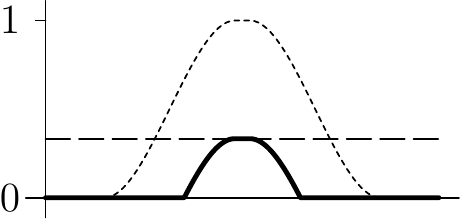} & 
      \includegraphics{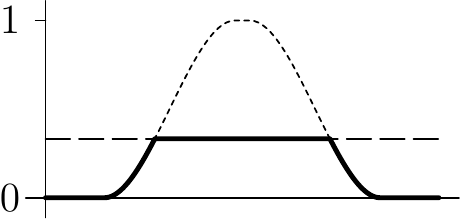} &
      \includegraphics{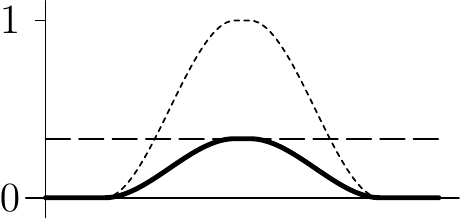}
      \\[1ex]
      \move{$\shf[c\rightarrow]$} &
      \includegraphics{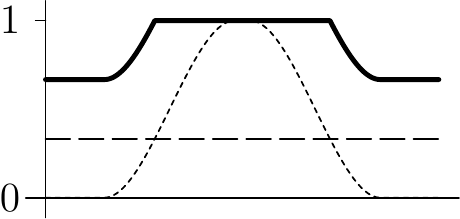} & 
      \includegraphics{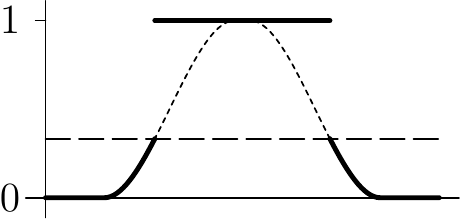} &
      \includegraphics{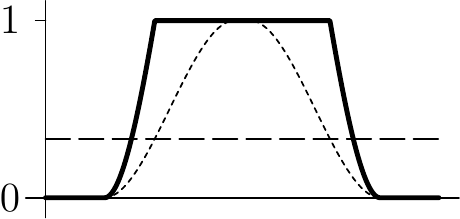}
      \\[1ex]
      \move{$\mul[\ominus c]$} &
      \includegraphics{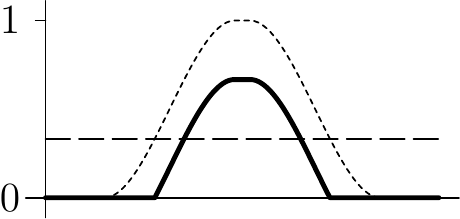} & 
      \includegraphics{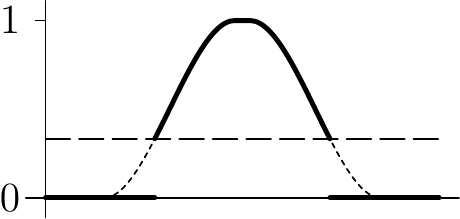} &
      \includegraphics{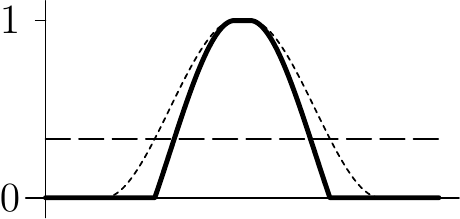}
      \\[1ex]
      \move{$\shf[c \oplus]$} &
      \includegraphics{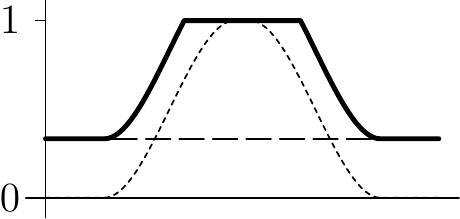} & 
      \includegraphics{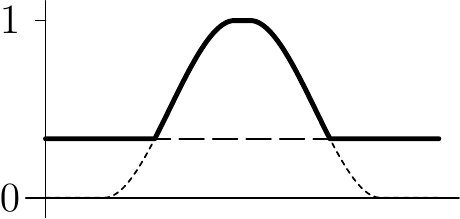} &
      \includegraphics{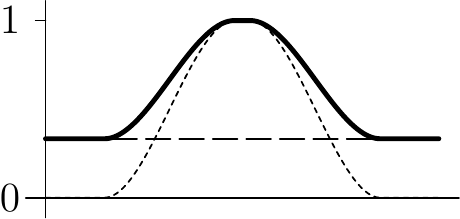} \\[-1ex]
      & \L ukasiewicz & G\"odel & Goguen
    \end{tabular}

    \caption{Examples of operators from Example~\ref{ex:S}
      for \L ukasiewicz, G\"odel, and Goguen operations.}
    \label{fig:mulshf}
  \end{figure}

  (f)
  So far, the parameterizations have been generated by \monoton e Galois
  connections arising by adjoint operations in $\mathbf{L}$ but our approach
  is far more general than that. For instance, let $Y = \{0,1,\ldots,n-1\}$
  and consider $S = \{\langle\mul[+i],\shf[-i]\rangle;\, i = 1,\ldots,n\}$,
  where 
  \begin{align}
    (\mul[+i]{A})(y) = A((y + i) \bmod n),
    \label{eqn:mul+}
    \\
    (\shf[-i]{B})(y) = B((y - i) \bmod n)
    \label{eqn:shf-}
  \end{align}
  for all $i=1,\ldots,n$, $A,B \in L^Y$, and $y \in Y$ with $x \bmod n$
  denoting the result of the usual modulo operation. Clearly,
  $\langle\mul[+n],\shf[-n]\rangle = \langle\one,\one\rangle \in S$ and
  $S$ is closed under $\circ$. Hence, $\mathbf{S}$ may be viewed as
  a parameterization which formalizes requirement on ``rotation of attributes''.
  Indeed, $M \models^\mathbf{S} A \I B$ if{}f, for every $i$:
  If $A$ rotated by $i$ (modulo $n$) is a subset of $M$, then 
  $B$ rotated by $i$ (modulo $n$) is a subset of $M$. This particular
  parameterization represents a non-trivial modification of the semantics of
  if-then rules even in the crisp case, i.e., when $\mathbf{L}$ is
  a two-element Boolean algebra.

  (g)
  All the methods (a)--(f) can be combined. In general, one can take
  existing parameterizations $\mathbf{S}_1,\mathbf{S}_2,\ldots$ and
  consider a parameterization which is generated by the union of sets
  of \monoton e Galois connections in all $\mathbf{S}_1,\mathbf{S}_2,\ldots$
  Further possibilities of obtaining parameterizations is to generate them
  from existing object-attribute data with graded attributes~\cite{GePo:Ndfc}.
\end{example}

In the rest of this section, we study the structure of models of FAIs with
general parameterizations and properties of the semantic entailment
under $\mathbf{S}$. We first show that systems of models of FAIs parameterized
by $\mathbf{S}$ are exactly closure systems which are in addition closed under
applications of all upper adjoints in $\mathbf{S}$.

\begin{definition}
  An operator $\C\!: L^Y \to L^Y$ is called an
  \emph{$\mathbf{S}$-closure operator in $\langle L^Y, \subseteq\rangle$}
  whenever
  \begin{align}
    A &\subseteq \C{A},
    \label{eqn:cl_ext} \\
    A \subseteq B &\text{ implies } \C{A} \subseteq \C{B},
    \label{eqn:cl_mon} \\
    \C{\shf{\C{A}}} &\subseteq \shf{\C{A}},
    \label{eqn:cl_shf}
  \end{align}
  are satisfied for all $A,B \in L^Y$ and
  all $\langle \mul,\shf\rangle \in S$.
  A system $\mathcal{S} \subseteq L^Y$ is called an
  \emph{$\mathbf{S}$-closure system in $\langle L^Y, \subseteq\rangle$}
  whenever it is closed under arbitrary intersections and
  \begin{align}
    M \in \mathcal{S}
    \text{ implies } \shf{M} \in \mathcal{S}
  \end{align}
  for each $M \in \mathcal{S}$ and all $\langle \mul,\shf\rangle \in S$.
\end{definition}

Note that $\mathbf{S}$-closure operators are indeed closure operators
because the idempotency condition $\C{\C{A}} \subseteq \C{A}$ is a special
case of \eqref{eqn:cl_shf} for $\shf = \one$.

\begin{theorem}\label{th:clos_sysop}
  Let $\C$ and $\mathcal{S}$ be an $\mathbf{S}$-closure operator and
  an $\mathbf{S}$-closure system in $\langle L^Y, \subseteq\rangle$,
  respectively. Then, $\mathcal{S}_{\C} = \{A \in L^Y;\, A = \C{A}\}$
  and $\C[\mathcal{S}]$ where
  $\C[\mathcal{S}]{A} = \bigcap\{B \in \mathcal{S};\, A \subseteq B\}$
  for all $A \in L^Y$ are an $\mathbf{S}$-closure system and
  an $\mathbf{S}$-closure operator in $\langle L^Y, \subseteq\rangle$,
  respectively. In addition to that,
  $\C = \C[\mathcal{S}_{\C}]$ and 
  $\mathcal{S} = \mathcal{S}_{\C[\mathcal{S}]}$.
\end{theorem}
\begin{proof}
  Take $A \in \mathcal{S}_{\C}$. Using~\eqref{eqn:cl_shf},
  $\C{\shf{A}} = \C{\shf{\C{A}}} \subseteq \shf{\C{A}} = \shf{A}$.
  The converse inclusion follows by~\eqref{eqn:cl_ext}, i.e.,
  $\C{\shf{A}} = \shf{A}$, showing $\shf{A} \in \mathcal{S}_{\C}$.
  Also, $\C[\mathcal{S}]{A} \in \mathcal{S}$ and thus
  $\shf{\C[\mathcal{S}]{A}} \in \mathcal{S}$
  which gives $\shf{\C[\mathcal{S}]{A}} =
  \C[\mathcal{S}]{\shf{\C[\mathcal{S}]{A}}}$, proving~\eqref{eqn:cl_shf}.
  The rest is clear.
\end{proof}

\begin{theorem}\label{th:Mod_clos}
  Let $\Sigma$ be a set of FAIs. Then $\mathrm{Mod}^\mathbf{S}(\Sigma)$ is
  an $\mathbf{S}$-closure system.
\end{theorem}
\begin{proof}
  The fact that $\mathrm{Mod}^\mathbf{S}(\Sigma)$ is closed under arbitrary
  intersections follows by standard arguments.
  We show that $\mathrm{Mod}^\mathbf{S}(\Sigma)$ is
  closed under all $\shf$'s. That is, for any
  $M \in \mathrm{Mod}^\mathbf{S}(\Sigma)$ and $A \I B \in \Sigma$,
  we prove that $\shf{M} \models A \I B$ for all $\shf$ in $\mathbf{S}$.
  We utilize the fact that $S$ is closed under composition.
  Let $\langle \shf[1],\mul[1]\rangle \in S$ and 
  $\langle \shf[2],\mul[2]\rangle \in S$.
  If $\mul[2](A) \subseteq \shf[1]{M}$, then using~\eqref{eqn:gal},
  we get $\mul[1]\mul[2](A) = \mul[1]{\mul[2](A)} \subseteq M$
  and so $\mul[1]\mul[2](B) = \mul[1]{\mul[2](B)} \subseteq M$
  because $\mul[1]\mul[2]$ is a composed operator in $\mathbf{S}$ and
  $M \models A \I B$.
  Therefore, \eqref{eqn:gal} used once again yields 
  $\mul[2](B) \subseteq \shf[1]{M}$, meaning that $\shf[1]{M} \models A \I B$
  and so $\shf[1]{M} \in \mathrm{Mod}^\mathbf{S}(\Sigma)$.
\end{proof}

\begin{theorem}\label{th:clos_Mod}
  Let $\mathcal{S}$ be an $\mathbf{S}$-closure system
  in $\langle L^Y, \subseteq\rangle$. Then, for
  \begin{align}
    \Sigma_\mathcal{S} = \{A \I \C[\mathcal{S}]{A};\, A \in L^Y\},
  \end{align}
  we have $\mathcal{S} = \mathrm{Mod}^\mathbf{S}(\Sigma_\mathcal{S})$.
\end{theorem}
\begin{proof}
  Recall that $\C[\mathcal{S}]$ is the $\mathbf{S}$-closure
  operator induced by $\mathcal{S}$, see Theorem~\ref{th:clos_sysop}.
  We prove the assertion by showing that
  both inclusions of
  $\mathcal{S} = \mathrm{Mod}^\mathbf{S}(\Sigma_\mathcal{S})$ hold.

  Let $M \in \mathcal{S}$ and $A \in L^Y$. That is,
  $M = \C[\mathcal{S}]{M}$ and $A \I \C[\mathcal{S}]{A} \in \Sigma_\mathcal{S}$.
  Furthermore, consider any $\langle\mul,\shf\rangle \in S$.
  If $\mul{A} \subseteq M$,
  then $\mul{A} \subseteq \C[\mathcal{S}]{M}$ and so
  $A \subseteq \shf{\C[\mathcal{S}]{M}}$.
  Using the \monoton y of $\C[\mathcal{S}]$, we further get
  $\C[\mathcal{S}]{A} \subseteq \C[\mathcal{S}]{\shf{\C[\mathcal{S}]{M}}}$
  and so~\eqref{eqn:cl_shf} yields
  $\C[\mathcal{S}]{A} \subseteq \shf{\C[\mathcal{S}]{M}}$, i.e.,
  $\mul{\C[\mathcal{S}]{A}} \subseteq \C[\mathcal{S}]{M} = M$.
  Therefore, for an operator $\mul$, we have shown that,
  for any $A \in L^Y$,
  $\mul{A} \subseteq M$ implies
  $\mul{\C[\mathcal{S}]{A}} \subseteq M$, i.e.,
  $M \in \mathrm{Mod}^\mathbf{S}(\Sigma_\mathcal{S})$.

  Conversely, assume that $M \in \mathrm{Mod}^\mathbf{S}(\Sigma_\mathcal{S})$.
  It suffices to show that $M$ is a fixed point of $\C[\mathcal{S}]$. This
  is easy to see since from $M \I \C[\mathcal{S}]{M} \in \Sigma_\mathcal{S}$
  and considering $\mul = \shf = \one$,
  it follows that $\one{M} \subseteq M$ and so
  $\C[\mathcal{S}]{M} = \one{\C[\mathcal{S}]{M}} \subseteq M$,
  proving that $M = \C[\mathcal{S}]{M}$.
\end{proof}

The following assertion characterizes the semantic entailment under $\mathbf{S}$
in terms of least models. Since $\mathrm{Mod}^\mathbf{S}(\Sigma)$ is an
$\mathbf{S}$-closure system, Theorem~\ref{th:clos_sysop}
allows us to consider the corresponding
$\mathbf{S}$-closure operator $\C[\mathrm{Mod}^\mathbf{S}(\Sigma)]$. For
brevity, we denote the operator $\C[\mathrm{Mod}^\mathbf{S}(\Sigma)]$
simply by $[{\cdots}]^\mathbf{S}_\Sigma$, i.e., $[A]^\mathbf{S}_\Sigma$ is
the \emph{least $\mathbf{S}$-model} of $\Sigma$ \emph{containing $A$.}

\begin{theorem}\label{th:sement}
  For each set $\Sigma$ of FAIs and each $A \I B$, the following
  conditions are equivalent:
  \begin{enumerate}\parskip=-2pt
  \item[\itm{1}]
    $\Sigma \models^\mathbf{S} A \I B$,
  \item[\itm{2}]
    $[M]^\mathbf{S}_\Sigma \models^\mathbf{S} A \I B$ for all $M \in L^Y$,
  \item[\itm{3}]
    $[A]^\mathbf{S}_\Sigma \models^\mathbf{S} A \I B$,
  \item[\itm{4}]
    $B \subseteq [A]^\mathbf{S}_\Sigma$,
  \end{enumerate}
\end{theorem}
\begin{proof}
  If \itm{1} holds, then \itm{2} is satisfied because
  $[M]^\mathbf{S}_\Sigma \in \mathrm{Mod}^\mathbf{S}(\Sigma)$
  for all $M \in L^Y$. In addition to that, \itm{2} implies \itm{3}
  trivially. Furthermore, for $\mul = \shf = \one$,
  we have $\mul{A} = A \subseteq [A]^\mathbf{S}_\Sigma$ and
  by \itm{3} it follows that $B = \mul{B} \subseteq [A]^\mathbf{S}_\Sigma$,
  showing \itm{4}. So, it suffices to check that \itm{4} implies \itm{1}.

  Take $M \in \mathrm{Mod}^\mathbf{S}(\Sigma)$ and let $\mul{A} \subseteq M$
  for $\langle\mul,\shf\rangle \in S$. As a consequence,
  $A \subseteq \shf{M}$ and thus
  $[A]^\mathbf{S}_\Sigma \subseteq [\shf{M}]^\mathbf{S}_\Sigma$ because of
  the \monoton y of $[{\cdots}]^\mathbf{S}_\Sigma$. Now, \itm{4} yields
  $B \subseteq [\shf{M}]^\mathbf{S}_\Sigma$. In addition to that,
  Theorem~\ref{th:Mod_clos} shows that
  $\shf{M} \in \mathrm{Mod}^\mathbf{S}(\Sigma)$
  and so $[\shf{M}]^\mathbf{S}_\Sigma = \shf{M}$,
  showing $B \subseteq \shf{M}$ and thus $\mul{B} \subseteq M$,
  proving $M \models^\mathbf{S} A \I B$ which establishes~\itm{1}.
\end{proof}

The semantic entailment of FAIs parameterized by hedges has the
following property: For any set $\Sigma$ of FAIs and any $A,B \in L^Y$,
$||A \I B||^*_\Sigma = 1$ if and only if 
$||0_Y \I B||^*_{\Sigma \cup \{0_Y \I A\}} = 1$.
This property can be seen as a semantic counterpart to the classic
deduction theorem of propositional logic. The following assertion shows
that the property holds for the general semantics if all $\mul$'s
are intensive, i.e., $\mul{M} \subseteq M$ for all $\mul$ and $M$.

\begin{theorem}\label{th:semded}
  Let $\mul{M} \subseteq M$ for all $\langle\mul,\shf\rangle \in S$
  and $M \in L^Y$. Then, for any $\Sigma$ and $A,B \in L^Y$,
  we have $\Sigma \models^\mathbf{S} A \I B$ if{}f\/
  $\Sigma \cup \{0_Y \I A\} \models^\mathbf{S} 0_Y \I B$.
\end{theorem}
\begin{proof}
  The only-if part follows by the monotony of $\models^\mathbf{S}$.
  In order to prove the if-part of the assertion,
  assume that $\Sigma \cup \{0_Y \I A\} \models^\mathbf{S} 0_Y \I B$
  and take any $M \in \mathrm{Mod}^\mathbf{S}(\Sigma)$. Furthermore,
  suppose that $\mul[1]{A} \subseteq M$ for
  $\langle\mul[1],\shf[1]\rangle \in S$.
  It follows that $A \subseteq \shf[1]{M}$. In addition, for any
  $\langle\mul[2],\shf[2]\rangle \in S$, $\mul[2]$ is \monoton e and so
  $\mul[2]{A} \subseteq \mul[2]{\shf[1]{M}}$. Using the assumption of
  intensivity of $\mul[2]$, the last inequality yields
  $\mul[2]{A} \subseteq \shf[1]{M}$. That is,
  $\shf[1]{M} \models^\mathbf{S} 0_Y \I A$ because $\mul[2]$
  has been taken arbitrarily. Moreover,
  Theorem~\ref{th:Mod_clos} shows that
  $\shf[1]{M} \in \mathrm{Mod}^\mathbf{S}(\Sigma)$ and so
  $\shf[1]{M} \in \mathrm{Mod}^\mathbf{S}(\Sigma \cup \{0_Y \I A\})$
  which further gives $\shf[1]{M} \models^\mathbf{S} 0_Y \I B$
  because $\Sigma \cup \{0_Y \I A\} \models^\mathbf{S} 0_Y \I B$.
  Hence, for $\mul = \one$, $\shf[1]{M} \models^\mathbf{S} 0_Y \I B$
  yields $B \subseteq \shf[1]{M}$. Therefore, we have shown that 
  $\mul[1]{A} \subseteq M$ implies $\mul[1]{B} \subseteq M$ for all
  $\langle\mul[1],\shf[1]\rangle \in S$ and
  $M \in \mathrm{Mod}^\mathbf{S}(\Sigma)$,
  proving $\Sigma \models^\mathbf{S} A \I B$.
\end{proof}

According to Theorem~\ref{th:sement},
in order to check $\Sigma \models^\mathbf{S} A \I B$,
it suffices to determine $[A]^\mathbf{S}_\Sigma$ and check whether the
inclusion $B \subseteq [A]^\mathbf{S}_\Sigma$ is satisfied. Constructive
methods to compute fixed points of $[{\cdots}]^\mathbf{S}_\Sigma$ can be
introduced based on computing fixed points of
immediate consequence operators~\cite{Ta:Altfta,KoEm:Splpl}.
For any $\Sigma$ and $\mathbf{S}$, we define an operator
$\boldsymbol{t}^\mathbf{S}_\Sigma\!: L^Y \to L^Y$ by
\begin{align}
  \boldsymbol{t}^\mathbf{S}_\Sigma(M) &= M \cup \textstyle\bigcup\{\mul{B};\,
  A \I B \in \Sigma, \langle \mul,\shf\rangle \in S
  \text{, and } \mul{A} \subseteq M\}.
  \label{eqn:t_op}
\end{align}
for all $M \in L^Y$. The operator is \monoton e and extensive and by standard
arguments it follows that fixed points of $[{\cdots}]^\mathbf{S}_\Sigma$ may
be obtained as fixed points of an iterated closure operator based
on~\eqref{eqn:t_op}. In particular, if both $Y$ and $L$ are finite, then
there is $N$ such that $N = \boldsymbol{t}^\mathbf{S}_\Sigma(
\boldsymbol{t}^\mathbf{S}_\Sigma(\cdots
(M)\cdots))$ with $\boldsymbol{t}^\mathbf{S}_\Sigma$ applied at most
$|L| \times |Y|$ times for which we have
$\boldsymbol{t}^\mathbf{S}_\Sigma(N) = N$ and thus
$[M]^\mathbf{S}_\Sigma = N$,
see also \cite{DaPe:MaRLP,Lloyd84}. This observation allows to use
a simple modification of the well-known algorithm
\textsc{Closure}~\cite[Algorithm 4.2]{Mai:TRD} to compute the
fixed points of $[{\cdots}]^\mathbf{S}_\Sigma$, cf. also~\cite{GaWi:FCA}.

\section{Description of Dependencies in Data}\label{sec:ddd}
In this section, we describe FAIs which are true in given object-attribute
data with fuzzy attributes and characterize non-redundant sets of FAIs which
describe all FAIs true in given data. The input data can be seen as
two-dimensional tables with rows corresponding to objects, columns corresponding
to attributes, and table entries being degrees in $L$, incidating degrees to
which objects have do/not have attributes, i.e., we work with the same type
of input data as in~\cite{BeVy:Fcalh} and related approaches.
The input data is formalized as follows.

For a non-empty set $X$ of \emph{objects} and set $Y$ of attributes (as before),
an \emph{$\mathbf{L}$-context} (a fuzzy context with degrees in $\mathbf{L}$,
see~\cite{Bel:FRS,BeVy:Fcalh,Po:FB}) is a triplet $\langle X, Y, I\rangle$ where
$I\!: X \times Y \to L$, i.e., $I$ is a binary $\mathbf{L}$-relation between
$X$ and $Y$; $I(x,y) \in L$ is interpreted as a degree to which the object
$x \in X$ has the attribute $y \in Y$. In order to simplify notation,
for any $x \in X$ we consider $I_x \in L^Y$ such that $I_x(y) = I(x,y)$
for all $y \in Y$. Under this notation, we define the notion of $A \I B$
being true in $\langle X,Y,I\rangle$ under $\mathbf{S}$ as follows.

\begin{definition}
  Let $\langle X,Y,I\rangle$ be an $\mathbf{L}$-context and let $A,B \in L^Y$.
  We say that
  \emph{$A \I B$ is true in $\langle X,Y,I\rangle$},
  written $I \models^\mathbf{S} A \I B$,
  whenever $I_x \models^\mathbf{S} A \I B$ for all $x \in X$.
\end{definition}

Our goal is to characterize, in a concise way,
all FAIs which are true in given $\langle X,Y,I\rangle$
considering $\mathbf{S}$.
The description we
offer here utilizes a couple of operators
${}^{\up}\!: 2^{X \times S_g} \to L^Y$ and
${}^{\dn}\!: L^Y \to 2^{X \times S_g}$ where 
$S_g = \{\shf;\, \langle\mul,\shf\rangle \in S\}$
such that
\begin{align}
  F^\up &= \textstyle\bigcap\{\shf{I_x};\, \langle x,\shf\rangle \in F\},
  \label{eqn:up}
  \\
  G^\dn &= \{\langle x,\shf\rangle;\, G \subseteq \shf{I_x}\},
  \label{eqn:dn}
\end{align}
for all $F \subseteq X \times S_g$ and $G \in L^Y$. It is easy to see
that $\langle {}^{\up},{}^{\dn}\rangle$ forms an
\emph{antitone Galois connection}~\cite{GaWi:FCA}, i.e.,
$F \subseteq G^{\dn}$ if{}f $G \subseteq F^{\up}$ for all
$F \subseteq X \times S_g$ and $G \in L^Y$. As a consequence,
the composed operator ${}^{\dn\up}\!: L^Y \to L^Y$, i.e.
\begin{align}
  G^{\dn\up} &= \textstyle\bigcap\{\shf{I_x};\, G \subseteq \shf{I_x}\},
  \label{eqn:dnup}
\end{align}
for all $G \in L^Y$, is a closure operator. The following assertion shows
that $G^{\dn\up}$ can be seen as an $\mathbf{L}$-set of attributes which are
implied by $G$ and can be used to characterize FAIs which are true
in $\langle X,Y,I\rangle$ under $\mathbf{S}$.

\begin{theorem}\label{th:Itrue}
  For each $I\!: X \times Y \to L$ and each $A \I B$,
  the following conditions are equivalent:
  \begin{enumerate}\parskip=-2pt
  \item[\itm{1}]
    $I \models^\mathbf{S} A \I B$,
  \item[\itm{2}]
    $M^{\up\dn} \models^\mathbf{S} A \I B$ for all $M \in L^Y$,
  \item[\itm{3}]
    $A^{\up\dn} \models^\mathbf{S} A \I B$,
  \item[\itm{4}]
    $B \subseteq A^{\up\dn}$,
  \item[\itm{5}]
    $A^{\dn} \subseteq B^{\dn}$.
  \end{enumerate}
\end{theorem}
\begin{proof}
  First, observe that ``\itm{2}\,$\Rightarrow$\,\itm{3}'' is trivial,
  ``\itm{3}\,$\Rightarrow$\,\itm{4}'' follows immediately for $\mul = \one$,
  and ``\itm{4}\,$\Rightarrow$\,\itm{5}'' is a consequence of the fact that
  $\langle {}^\up,{}^\dn\rangle$ is an antitone Galois connection. Hence,
  it remains to prove that \itm{1} implies \itm{2} and that
  \itm{5} implies \itm{1}.

  Suppose that \itm{1} is satisfied.
  Take $\langle\mul[1],\shf[1]\rangle \in S$ and $M \in L^Y$ such that
  $\mul[1]{A} \subseteq M^{\dn\up}$. Using~\eqref{eqn:dnup}, the last
  inclusion means that $\mul[1]{A} \subseteq \shf[2]{I_x}$ for all
  $x \in X$ and $\langle\mul[2],\shf[2]\rangle \in S$ such that
  $M \subseteq \shf[2]{I_x}$. Since $I \models^\mathbf{S} A \I B$,
  it then follows that $\mul[1]{B} \subseteq \shf[2]{I_x}$ for all
  $x \in X$ and $\langle\mul[2],\shf[2]\rangle \in S$ such that
  $M \subseteq \shf[2]{I_x}$. Hence, \eqref{eqn:dnup} gives
  $\mul[1]{B} \subseteq M^{\dn\up}$, proving~\itm{2}.

  Finally, suppose that \itm{5} is satisfied.
  Using~\eqref{eqn:dn},
  $A^{\dn} \subseteq B^{\dn}$ yields that for all $x \in X$ and
  $\langle\mul,\shf\rangle \in S$:
  $A \subseteq \shf{I_x}$ implies $B \subseteq \shf{I_x}$,
  i.e., $\mul{A} \subseteq I_x$ implies $\mul{B} \subseteq I_x$,
  proving~\itm{1}.
\end{proof}

The rest of this section is devoted to determining bases of FAIs. That is,
given $\langle X,Y,I\rangle$, we wish to find non-redudnant sets of FAIs
which entail exactly all FAIs which are true in $\langle X,Y,I\rangle$
under $\mathbf{S}$. In a similar way as in the case of parameterizations
by hedges~\cite[Section 5]{BeVy:ADfDwG}, we show that all properties
necessary to determine bases hold for any parameterization $\mathbf{S}$.

\begin{definition}
  Let $\langle X,Y,I\rangle$ be an $\mathbf{L}$-context.
  A set $\Sigma$ of FAIs is called
  \emph{$\mathbf{S}$-complete in $\langle X,Y,I\rangle$}
  whenever, for all $A,B \in L^Y$, $\Sigma \models^\mathbf{S} A \I B$ if{}f
  $I \models^\mathbf{S} A \I B$. Furthermore, $\Sigma$ is called an
  \emph{$\mathbf{S}$-base of $\langle X,Y,I\rangle$} if it is
  $\mathbf{S}$-complete in $\langle X,Y,I\rangle$ and no
  $\Sigma' \subset \Sigma$ is $\mathbf{S}$-complete in $\langle X,Y,I\rangle$.
\end{definition}

\begin{theorem}\label{th:dcompl}
  Let $\langle X,Y,I\rangle$ be an $\mathbf{L}$-context and $\Sigma$ be
  a set of FAIs in $Y$. Then, the following conditions are equivalent:
  \begin{enumerate}\parskip=-2pt
  \item[\itm{1}]
    $\Sigma$ is $\mathbf{S}$-complete in $\langle X,Y,I\rangle$,
  \item[\itm{2}]
    $\mathrm{Mod}^\mathbf{S}(\Sigma) = \mathcal{S}_{\dn\up}$,
  \item[\itm{3}]
    $[M]^\mathbf{S}_\Sigma = M^{\dn\up}$ for all $M \in L^Y$.
  \end{enumerate}
\end{theorem}
\begin{proof}
  Clearly, \itm{2} and \itm{3} are equivalent because
  $\mathrm{Mod}^\mathbf{S}(\Sigma)$ and $\mathcal{S}_{\dn\up}$ (the set of
  all fixed points of the closure operator ${}^{\dn\up}$) coincide
  if and only if the fixed points generated by any $M \in L^Y$ coincide.
  Furthermore, \itm{3} implies \itm{1}. Indeed, for any $A \I B$,
  using Theorem~\ref{th:sement}, we have $\Sigma \models^\mathbf{S} A \I B$
  if{}f $B \subseteq [A]^\mathbf{S}_\Sigma = A^{\dn\up}$ which is according
  to Theorem~\ref{th:Itrue} true if{}f $I \models^\mathbf{S} A \I B$,
  proving~\itm{1}. Therefore, it suffices to prove that \itm{1} implies \itm{3}.
  Take any $M \in L^Y$. Since $M^{\dn\up} \subseteq M^{\dn\up}$,
  Theorem~\ref{th:Itrue}
  gives $I \models^\mathbf{S} M \I M^{\dn\up}$ and so
  $\Sigma \models^\mathbf{S} M \I M^{\dn\up}$, showing
  $M^{\dn\up} \subseteq [M]^\mathbf{S}_\Sigma$ on account of
  Theorem~\ref{th:sement}. The converse inclusion can be proved in much
  the same way: $[M]^\mathbf{S}_\Sigma \subseteq [M]^\mathbf{S}_\Sigma$
  gives $\Sigma \models^\mathbf{S} M \I [M]^\mathbf{S}_\Sigma$ by
  Theorem~\ref{th:sement} and so we have
  $I \models^\mathbf{S} M \I [M]^\mathbf{S}_\Sigma$ which yields
  $[M]^\mathbf{S}_\Sigma \subseteq M^{\dn\up}$ owing to Theorem~\ref{th:Itrue}.
\end{proof}

\begin{theorem}\label{th:base}
  Let $\langle X,Y,I\rangle$ be an $\mathbf{L}$-context and $\Sigma$ be
  a set of FAIs which is $\mathbf{S}$-complete in~$\langle X,Y,I\rangle$.
  Then, the following conditions are equivalent:
  \begin{enumerate}\parskip=-2pt
  \item[\itm{1}]
    $\Sigma$ is an $\mathbf{S}$-base of $\langle X,Y,I\rangle$,
  \item[\itm{2}]
    $\Sigma \setminus \{A \I B\} \nmodels^\mathbf{S} A \I B$\/
    for all $A \I B \in \Sigma$,
  \item[\itm{3}]
    $[A]^\mathbf{S}_{\Sigma \setminus \{A \I B\}} \subset
    [A]^\mathbf{S}_\Sigma$\,
    for all $A \I B \in \Sigma$.
  \end{enumerate}
\end{theorem}
\begin{proof}
  In order to see that \itm{1} implies \itm{2}, take any $A \I B \in \Sigma$
  and observe that $\Sigma \setminus \{A \I B\}$ is not $\mathbf{S}$-complete in
  $\langle X,Y,I\rangle$. Hence, $\mathrm{Mod}^\mathbf{S}(\Sigma) \subset
  \mathrm{Mod}^\mathbf{S}(\Sigma \setminus \{A \I B\})$ on account of
  Theorem~\ref{th:dcompl}.
  Take $M \in \mathrm{Mod}^\mathbf{S}(\Sigma \setminus \{A \I B\})$
  such that $M \not\in \mathrm{Mod}^\mathbf{S}(\Sigma)$.
  We have $M \nmodels^\mathbf{S} A \I B$ because otherwise
  we would obtain $M \in \mathrm{Mod}^\mathbf{S}(\Sigma)$.
  Therefore, $\Sigma \setminus \{A \I B\} \nmodels^\mathbf{S} A \I B$.

  Now, assume that \itm{2} is satisfied.
  The fact $\Sigma \setminus \{A \I B\} \nmodels^\mathbf{S} A \I B$ means
  $B \nsubseteq [A]^\mathbf{S}_{\Sigma \setminus \{A \I B\}}$ owing to
  Theorem~\ref{th:sement}. Since $\Sigma \models^\mathbf{S} A \I B$ trivially
  because $A \I B \in \Sigma$, we get $B \subseteq [A]^\mathbf{S}_{\Sigma}$.
  Moreover, $[A]^\mathbf{S}_{\Sigma \setminus \{A \I B\}} \subseteq
  [A]^\mathbf{S}_{\Sigma}$ together with $B \subseteq [A]^\mathbf{S}_{\Sigma}$
  and $B \nsubseteq [A]^\mathbf{S}_{\Sigma \setminus \{A \I B\}}$ yield
  $[A]^\mathbf{S}_{\Sigma \setminus \{A \I B\}} \subset [A]^\mathbf{S}_\Sigma$,
  proving~\itm{3}.

  Finally, assume that \itm{3} is satisfied and let any
  $\Sigma' \subset \Sigma$.
  Take $A \I B \in \Sigma$ such that $A \I B \not\in \Sigma'$. We have
  $\Sigma \models^\mathbf{S} A \I [A]^\mathbf{S}_\Sigma$
  on account of Theorem~\ref{th:sement}. On the other hand,
  $[A]^\mathbf{S}_{\Sigma \setminus \{A \I B\}} \subset [A]^\mathbf{S}_\Sigma$
  means $[A]^\mathbf{S}_\Sigma \nsubseteq
  [A]^\mathbf{S}_{\Sigma \setminus \{A \I B\}}$ and so 
  $\Sigma \setminus \{A \I B\} \nmodels^\mathbf{S} A \I [A]^\mathbf{S}_\Sigma$
  by Theorem~\ref{th:sement}.
  As a consequence, $\Sigma' \nmodels^\mathbf{S} A \I [A]^\mathbf{S}_\Sigma$.
  Therefore, $\Sigma'$ is not $\mathbf{S}$-complete in $\langle X,Y,I\rangle$,
  proving \itm{1}.
\end{proof}

Particular sets of FAIs which are $\mathbf{S}$-complete in given data and
can be used to find bases by removing redundant formulas are given by systems
of $\mathbf{L}$-sets which are based on a generalized concept of
a pseudo-intent~\cite{GuDu}.

\begin{definition}\label{def:P}
  An $\mathbf{L}$-set $P \in L^Y$ is an
  \emph{$\mathbf{S}$-pseudo intent} of $\langle X,Y,I\rangle$ whenever
  $P \subset P^{\dn\up}$ and for each $\mathbf{S}$-pseudo
  intent $Q \subset P$ of $\langle X,Y,I\rangle$,
  we have $Q^{\dn\up} \subseteq P$.
\end{definition}

\begin{theorem}\label{th:pseudos}
  If\/ $\mathbf{L}$ and $Y$ are finite,
  then $\Sigma_I = \{P \I P^{\dn\up};\,
  P \text{ is $\mathbf{S}$-pseudo intent of $\langle X,Y,I\rangle$}\}$
  is $\mathbf{S}$-complete in $\langle X,Y,I\rangle$.
\end{theorem}
\begin{proof}
  As in the case of bivalent attribute implications~\cite{GaWi:FCA,GuDu},
  the finiteness of $\mathbf{L}$ and $Y$ ensures
  that $\mathbf{S}$-pseudo intents are well defined.
  Owing to Theorem~\ref{th:dcompl}, it suffices to check that 
  $\mathrm{Mod}^\mathbf{S}(\Sigma_I) = \mathcal{S}_{\dn\up}$. Evidently,
  we have $\mathcal{S}_{\dn\up} \subseteq \mathrm{Mod}^\mathbf{S}(\Sigma_I)$
  on account of Theorem~\ref{th:clos_Mod}. Thus, it suffices to prove the
  converse inclusion. Let $M \in \mathrm{Mod}^\mathbf{S}(\Sigma_I)$, i.e.,
  $M \models^\mathbf{S} P \I P^{\dn\up}$ for each $\mathbf{S}$-pseudo
  intent $P$ of $\langle X,Y,I\rangle$.
  Now, if $M$ were an $\mathbf{S}$-pseudo intent of $\langle X,Y,I\rangle$,
  we would get $M \not\in \mathrm{Mod}^\mathbf{S}(\Sigma_I)$
  since $M \nmodels^\mathbf{S} M \I M^{\dn\up}$. Therefore, $M$ is not
  an $\mathbf{S}$-pseudo intent of $\langle X,Y,I\rangle$.
  In addition, for every $\mathbf{S}$-pseudo intent
  $P \subset M$, $M \models^\mathbf{S} P \I P^{\dn\up}$ yields
  $P^{\dn\up} = \mul{P^{\dn\up}} \subseteq M$ for $\mul = \one$.
  Therefore, by Definition~\ref{def:P}, we must have $M = M^{\dn\up}$,
  i.e., $M \in \mathcal{S}_{\dn\up}$.
\end{proof}

Based on Theorem~\ref{th:pseudos}, we may determine an $\mathbf{S}$-base of
$\langle X,Y,I\rangle$ by first computing all $\mathbf{S}$-pseudo intents.
This can be done by any algorithm for computing fixed points of
fuzzy closure operators~\cite{BeBaOuVy:Lindig}
in lectical order~\cite{Ga:Tbaca}.
Then, Theorem~\ref{th:pseudos} yields that $\Sigma_I$
is complete. In case of $S = \{\langle\one,\one\rangle\}$, it can be shown
that it is in addition non-redundant and minimal in the number of
formulas~\cite[Theorem 5.20]{BeVy:ADfDwG}. This is a consequence of the fact
that for such $\mathbf{S}$, the semantics of FAIs corresponds to the
parameterization by globalization, see Example~\ref{ex:S}\,(a).
In general, $\Sigma_I$ is not an $\mathbf{S}$-base but applying
Theorem~\ref{th:base}\,\itm{2} and Theorem~\ref{th:sement}\,\itm{4},
we can determine its subset which is an $\mathbf{S}$-base by
removing all $P \I P^{\dn\up} \in \Sigma_I$ which are
redundant in $\Sigma_I$.

\begin{remark}\label{rem:clust}
  We have shown that the fixed points of ${}^{\dn\up}$ are useful in
  describing $\mathbf{S}$-bases of data. In addition, the fixed points
  may be seen as (fuzzy) clusters of attributes present in
  $\langle X,Y,I\rangle$. Indeed, following the usual interpretation
  of fixed points of concept-forming operators in formal concept
  analysis~\cite{GaWi:FCA}, $M^{\dn\up}$ is a (fuzzy) cluster of attributes
  (so-called \emph{intent} generated by $M$)
  shared by all objects $x \in X$ which have all the attributes in $M$;
  $M^{\dn\up}(y)$ is interpreted as the degree to
  which $y \in Y$ belongs to the cluster.
  When ordered by $\subseteq$ defined by~\eqref{eqn:crisp_sub},
  the set of all clusters in $\langle X,Y,I\rangle$ forms
  a complete lattice.
\end{remark}

\section{Complete Axiomatization and Approximate Inference}\label{sec:compl}
The semantic entailment under $\mathbf{S}$ is axiomatizable. Indeed, in this
section, we present a complete inference system and a particular notion of
provability which coincides with the semantic entailment under~$\mathbf{S}$.
Furthermore, in addition to the bivalent notion of
a semantic entailment under~$\mathbf{S}$,
we introduce its graded counterpart. That is, instead of
just considering $\Sigma \models^\mathbf{S} A \I B$ or 
$\Sigma \nmodels^\mathbf{S} A \I B$, we show there is a reasonable notion
of a degree to which $A \I B$ follows by $\Sigma$ under $\mathbf{S}$.
Interestingly, the degrees of semantic entailment can also be characterized
by a suitable notion of provability which can be derived from
the bivalent provability. In the following definition we utilize axioms
and the inference rule~\eqref{r:Cut} as they were presented in
Section~\ref{sec:pfai}.

\begin{definition}\label{def:synent}
  Let $\Sigma$ be a set of FAIs in $Y$ and $\mathbf{S}$ be a parameterization.
  An \emph{$\mathbf{S}$-proof of $A \I B$ by $\Sigma$}
  is a sequence $\varphi_1,\ldots,\varphi_n$ of
  FAIs such that $\varphi_n$ is $A \I B$ and, for every $i \in I$, $\varphi_i$
  is an axiom or $\varphi_i \in \Sigma$ or $\varphi_i$ results from some
  $\varphi_1,\ldots,\varphi_{i-1}$ using~\eqref{r:Cut} or using
  \begin{align}
    \dfrac{A \I B}{\mul{A} \I \mul{B}}
    \label{r:F}
  \end{align}
  for some $\langle\mul,\shf\rangle \in S$. If there is an
  $\mathbf{S}$-proof of $A \I B$ by $\Sigma$, we say that $A \I B$
  is \emph{$\mathbf{S}$-provable by $\Sigma$} and denote the fact by
  $\Sigma \vdash^\mathbf{S} A \I B$.
\end{definition}

The following soundness and completeness theorems are established.

\begin{theorem}\label{th:sound}
  If\/ $\Sigma \proves^\mathbf{S} A \I B$, then 
  $\Sigma \models^\mathbf{S} A \I B$.
\end{theorem}
\begin{proof}
  Observe we have
  $M \models^\mathbf{S} A{\cup}B \I B$ for any $M \in L^Y$.
  Indeed, if $\mul{A{\cup}B} \subseteq M$ for $\langle\mul,\shf\rangle \in S$
  then owing to the \monoton y of $\mul$ and transitivity of $\subseteq$,
  we get $\mul{B} \subseteq \mul{A{\cup}B} \subseteq M$.

  We show that~\eqref{r:Cut} is a sound inference rule.
  Let $M \models^\mathbf{S} A \I B$ and $M \models^\mathbf{S} B{\cup}C \I D$.
  Suppose that for $\langle\mul,\shf\rangle \in S$, we have
  $\mul{A{\cup}C} \subseteq M$.
  Then, we also have $\mul{A} \subseteq M$ and $\mul{C} \subseteq M$
  because $\mul$ is \monoton e. Hence, $\mul{A} \subseteq M$ and 
  $M \models^\mathbf{S} A \I B$ yield $\mul{B} \subseteq M$.
  Now, $\mul{B} \subseteq M$ together with $\mul{C} \subseteq M$ give
  $B \subseteq \shf{M}$ and $C \subseteq \shf{M}$ and thus
  $B \cup C \subseteq \shf{M}$, i.e., $\mul{B{\cup}C} \subseteq M$.
  Using $\mul{B{\cup}C} \subseteq M$, we get $\mul{D} \subseteq M$ because 
  $M \models^\mathbf{S} B{\cup}C \I D$. As a consequence, if 
  $M \models^\mathbf{S} A \I B$ and $M \models^\mathbf{S} B{\cup}C \I D$,
  then $M \models^\mathbf{S} A{\cup}C \I D$.

  Moreover, \eqref{r:F} is sound:
  Let $M \models^\mathbf{S} A \I B$,
  $\langle \mul[1],\shf[1]\rangle \in S$,
  and $\langle \mul[2],\shf[2]\rangle \in S$.
  Clearly, if $\mul[1]{\mul[2]{A}} = \mul[1]\mul[2]{A} \subseteq M$,
  then $\mul[1]{\mul[2]{B}} = \mul[1]\mul[2]{B} \subseteq M$ because
  $M \models^\mathbf{S} A \I B$ and $\mul[1]\mul[2]$ is a composed operator
  in $\mathbf{S}$. Therefore, $M \models^\mathbf{S} \mul{A} \I \mul{B}$
  for any $\langle\mul,\shf\rangle \in S$.

  The rest follows by induction on the length of an $\mathbf{S}$-proof.
\end{proof}

\begin{theorem}\label{th:compl}
  Let $Y$ and $\mathbf{L}$ be finite.
  If\/ $\Sigma \models^\mathbf{S} A \I B$, then 
  $\Sigma \proves^\mathbf{S} A \I B$.
\end{theorem}
\begin{proof}
  Suppose that $\Sigma \nproves^\mathbf{S} A \I B$, we show that 
  $\Sigma \nmodels^\mathbf{S} A \I B$. In order to see that,
  we find an $\mathbf{S}$-model of $\Sigma$ in which $A \I B$ is not true.
  Put $\mathcal{S}_A = \{C \in L^Y;\, \Sigma \proves^\mathbf{S} A \I C\}$
  and take $A^+ = \bigcup\mathcal{S}_A$.
  Since $\mathcal{S}_A$ is finite, using
  additivity~\cite[Lemma 4.2]{BeVy:ADfDwG} which is
  a consequence of having~\eqref{r:Cut} as our inference rule, we get
  that $A^+ \in \mathcal{S}_A$ and so $\Sigma \proves^\mathbf{S} A \I A^+$.
  Take any $E \I F \in \Sigma$ and suppose that $\mul{E} \subseteq A^+$
  for $\langle \mul,\shf\rangle \in S$.
  By projectivity~\cite[Lemma 4.2]{BeVy:ADfDwG},
  we get $\Sigma \proves^\mathbf{S} A \I \mul{E}$.
  Moreover, using $\Sigma \proves^\mathbf{S} E \I F$, we get
  $\Sigma \proves^\mathbf{S} \mul{E} \I \mul{F}$ by~\eqref{r:F}.
  By~\eqref{r:Cut},
  $\Sigma \proves^\mathbf{S} A \I \mul{F}$ which means $\mul{F} \subseteq A^+$.
  Therefore, $A^+ \in \mathrm{Mod}^\mathbf{S}(\Sigma)$.

  In addition to that, we have $A^+ \nmodels^\mathbf{S} A \I B$. Indeed,
  by contradiction, $A^+ \models^\mathbf{S} A \I B$ would yield
  $A = \one{A} \subseteq A^+$ and so $B = \one{B} \subseteq A^+$,
  i.e., $\Sigma \proves^\mathbf{S} A \I B$ by projectivity which contradicts 
  the fact that $\Sigma \nvdash^\mathbf{S} A \I B$.
\end{proof}

\begin{remark}\label{rem:CutF}
  Theorem~\ref{th:compl} is limited to finite $Y$ and $\mathbf{L}$. If one
  wishes to have a complete axiomatization for any $Y$ and $\mathbf{L}$,
  it can be done by introducing an \emph{infinitary cut},
  see~\cite{KuVy:Flprrai} and~\cite{BeVy:Falcrl} for details.
  Also note that there are several other inference systems which are
  equivalent to the system we use in this section. For instance, the
  inference rules~\eqref{r:Cut} and~\eqref{r:F} can equivalently
  be replaced by a single rule of the form
  \begin{align}
    \dfrac{A \I \mul{B},\,B \cup C \I D}{A \cup \mul{C} \I \mul{D}}
    \label{r:CutF}
  \end{align}
  for all $A,B,C,D \in L^Y$ and $\langle\mul,\shf\rangle \in S$.
  This is easy to see using \eqref{eqn:mul_distr} and $\mul{0_Y} = 0_Y$.
\end{remark}

Another equivalent inference system may be introduced by considering
normalized proofs using inference rules of reflexivity, accumulation,
and projectivity together with~\eqref{r:F} analogously as it is shown
in~\cite{BeVy:MRAP}. In fact, in order to adopt the approach
in~\cite{BeVy:MRAP} to our setting,
it suffices to prove that~\eqref{r:F} is idempotent and
commutes with axioms and~\eqref{r:Cut} in the following sense.

\begin{lemma}\label{le:commut}
  Each FAI which is derived using~\eqref{r:F} from an axiom is an axiom.
  If a FAI is derived first by using~\eqref{r:Cut}
  and then by using~\eqref{r:F},
  it can also be derived first by using \eqref{r:F} twice and then
  by using~\eqref{r:Cut}.
\end{lemma}
\begin{proof}
  Both claims follow by~\eqref{eqn:mul_distr}. Indeed, the first claim is
  immediate and the second one can be shown as follows.
  If a formula is derived first by
  \eqref{r:Cut} from $A \I B$ and $B \cup C \I D$ and then by~\eqref{r:F},
  then it must be of the form $\mul{A \cup C} \I \mul{D}$.
  Observe that~\eqref{r:F} used with $A \I B$ and $B \cup C \I D$
  yields $\mul{A} \I \mul{B}$ and $\mul{B \cup C} \I \mul{D}$ which is
  equal to $\mul{B} \cup \mul{C} \I \mul{D}$.
  Therefore, we may use~\eqref{r:Cut}
  to infer $\mul{A} \cup \mul{C} \I \mul{D}$ which equals 
  $\mul{A \cup C} \I \mul{D}$.
\end{proof}

As a consequence of Lemma~\ref{le:commut}, each $\mathbf{S}$-proof by $\Sigma$
can be transformed into an $\mathbf{S}$-proof of the same formula in which all
applications of~\eqref{r:F} appear before all applications of~\eqref{r:Cut}.
In the transformed proof, \eqref{r:F} is applied only to formulas in $\Sigma$.
We therefore have the following consequence.

\begin{corollary}
  $\Sigma \vdash^\mathbf{S} A \I B$ if{}f there is 
  \begin{align}
    \Sigma^\mathbf{S} \subseteq
    \{\mul{A} \I \mul{B};\, A \I B \in \Sigma \text{ and }
    \langle\mul,\shf\rangle \in S\}
    \label{eqn:SigmaS}
  \end{align}
  such that $A \I B$ is provable by $\Sigma^\mathbf{S}$ 
  using axioms and \eqref{r:Cut} as the only inference rule.
  \qed
\end{corollary}

In the rest of this section, we deal with graded notions of semantic
entailment and provability using general parameterization $\mathbf{S}$.
Recall that $\models^\mathbf{S}$
introduced in Definition~\ref{def:truth} is a bivalent notion. A formula
either is true in $M$ (or entailed by $\Sigma$) or not. In contrast, the
notions of truth and entailment of FAIs parameterized by
hedges~\cite{BeVy:ADfDwG}, i.e., \eqref{eqn:hedge_fai} and \eqref{eqn:sement*},
are introduced as graded notions. We now show that our general approach also
admits such graded notions. Interestingly, the introduced notions are fully
expressible by the bivalent ones.

\begin{definition}
  Let $A,B,M \in L^Y$ and let $\mathbf{S}$ be a parameterization of FAIs.
  The \emph{degree to which $A \I B$ is true in $M$ under $\mathbf{S}$},
  written $||A \I B||^\mathbf{S}_M$, is defined by
  \begin{align}
    ||A \I B||^\mathbf{S}_M &=
    \textstyle
    \bigvee\{c \in L;\, M \models^\mathbf{S} A \I c{\otimes}B\}.
    \label{eqn:grad_true}
  \end{align}
  Let $\Sigma$ be a set of FAIs.
  The \emph{degree to which $A \I B$ is semantically entailed by $\Sigma$
    under $\mathbf{S}$},
  written $||A \I B||^\mathbf{S}_\Sigma$, is defined by
  \begin{align}
    ||A \I B||^\mathbf{S}_\Sigma &=
    \textstyle
    \bigwedge_{M \in \mathrm{Mod}^\mathbf{S}(\Sigma)}||A \I B||^\mathbf{S}_M.
    \label{eqn:grad_sement}
  \end{align}
\end{definition}

\begin{remark}
  Two remarks are in order. First, $||A \I B||^\mathbf{S}_M \in L$
  is not only a supremum of degrees but also
  the greatest degree among all $c \in L$ such that
  $M \models^\mathbf{S} A \I c{\otimes}B$. Indeed, put
  $K = \{c \in L;\, M \models^\mathbf{S} A \I c{\otimes}B\}$
  and observe that trivially $0 \in K$. Moreover, if $c_i \in K$ ($i \in I$),
  then $\mul{c_i{\otimes}B} \subseteq M$ for all $i \in I$ yields 
  $c_i{\otimes}B \subseteq \shf{M}$ for all $i \in I$ and so
  $c{\otimes}B \subseteq \shf{M}$ for $c = \bigvee_{\!i \in I}c_i$
  owing to~\eqref{eqn:obigvee} and so $c \in K$,
  proving $||A \I B||^\mathbf{S}_M \in K$.
  Second, if $\mathbf{S}^*$ corresponds to a parameterization given by 
  hedge ${}^*$ as in Example~\ref{ex:S}\,(b),
  then $||A \I B||^{\mathbf{S}^*}_M = ||A \I B||^*_M$ for all $A,B,M \in L^Y$
  and the same applies to~\eqref{eqn:grad_sement} and \eqref{eqn:sement*}.
  This shows that graded entailment under $\mathbf{S}$ is a proper
  generalization of the graded entailment parameterized
  by hedges~\cite{BeVy:ADfDwG}.
\end{remark}

The following assertion shows that the least $\mathbf{S}$-model
$[A]^\mathbf{S}_\Sigma$ containing $A$ can be used to express
the degrees of semantic entailment under $\mathbf{S}$.
Therefore, the assertion extends Theorem~\ref{th:sement}
in that it characterizes arbitrary \emph{degrees of entailment}
and is not restricted just to the ``full entailment', i.e,
the entailment to degree $1$.

\begin{theorem}
  Let $\Sigma$ be a set of FAIs in $Y$, $\mathbf{S}$ be a parameterization.
  Then, for any $A,B \in L^Y$,
  \begin{align}
    ||A \I B||^\mathbf{S}_\Sigma &=
    \textstyle
    \bigvee\{c \in L;\, \Sigma \models^\mathbf{S} A \I c{\otimes}B\}
    = \SD\bigl(B,[A]^\mathbf{S}_\Sigma\bigr).
  \end{align}
\end{theorem}
\begin{proof}
  Using~\eqref{eqn:grad_sement},
  $[A]^\mathbf{S}_\Sigma \in \mathrm{Mod}^\mathbf{S}(\Sigma)$,
  \eqref{eqn:grad_true}, and Theorem~\ref{th:sement}\,\itm{3}, we have
  \begin{align*}
    ||A \I B||^\mathbf{S}_\Sigma
    &=
    \textstyle
    \bigwedge_{M \in \mathrm{Mod}^\mathbf{S}(\Sigma)}||A \I B||^\mathbf{S}_M
    \leq
    ||A \I B||^\mathbf{S}_{[A]^\mathbf{S}_\Sigma}
    \\
    &= 
    \textstyle
    \bigvee\{c \in L;\,
    [A]^\mathbf{S}_\Sigma \models^\mathbf{S} A \I c{\otimes}B\}
    = 
    \textstyle
    \bigvee\{c \in L;\, \Sigma \models^\mathbf{S} A \I c{\otimes}B\}.
  \end{align*}
  Using Theorem~\ref{th:sement}\,\itm{4}, it follows that
  \begin{align*}
    \textstyle
    \bigvee\{c \in L;\, \Sigma \models^\mathbf{S} A \I c{\otimes}B\}
    &= 
    \textstyle
    \bigvee\{c \in L;\, c{\otimes}B \subseteq [A]^\mathbf{S}_\Sigma\} 
    \\
    &=
    \textstyle
    \bigvee\{c \in L;\, c \leq \SD(B,[A]^\mathbf{S}_\Sigma)\}
    =
    \SD(B,[A]^\mathbf{S}_\Sigma).
  \end{align*}
  In order to prove that 
  $\SD(B,[A]^\mathbf{S}_\Sigma) \leq ||A \I B||^\mathbf{S}_\Sigma$,
  we show $\SD(B,[A]^\mathbf{S}_\Sigma) \leq ||A \I B||^\mathbf{S}_M$
  for any $M \in \mathrm{Mod}^\mathbf{S}(\Sigma)$.
  By~\eqref{eqn:grad_true}, it means showing  
  $\SD(B,[A]^\mathbf{S}_\Sigma) \leq
  \bigvee\{c \in L;\, M \models^\mathbf{S} A \I c{\otimes}B\}$ for any 
  $M \in \mathrm{Mod}^\mathbf{S}(\Sigma)$.
  For every $M \in \mathrm{Mod}^\mathbf{S}(\Sigma)$, it suffices to prove
  $M \models^\mathbf{S} A \I c {\otimes} B$
  for $c = \SD(B,[A]^\mathbf{S}_\Sigma)$ which is indeed the case:
  Assume that $\mul{A} \subseteq M$ for
  $\langle\mul,\shf\rangle \in S$. Then, $A \subseteq \shf{M}$
  and the \monoton y of $[{\cdots}]^\mathbf{S}_\Sigma$ yields
  $[A]^\mathbf{S}_\Sigma \subseteq [\shf{M}]^\mathbf{S}_\Sigma = \shf{M}$
  because $\shf{M} \in \mathrm{Mod}^\mathbf{S}(\Sigma)$
  owing to Theorem~\ref{th:Mod_clos}.
  Therefore, $\SD(B,[A]^\mathbf{S}_\Sigma) \leq \SD(B,\shf{M})$ which
  holds if{}f $\SD(B,[A]^\mathbf{S}_\Sigma) {\otimes} B \subseteq \shf{M}$,
  i.e., if{}f $\mul(\SD(B,[A]^\mathbf{S}_\Sigma) {\otimes} B) \subseteq M$,
  proving $M \models^\mathbf{S} A \I c {\otimes} B$
  for $c = \SD(B,[A]^\mathbf{S}_\Sigma)$.
\end{proof}

The previous observation allows us to define a
\emph{degree $|A \I B|^\mathbf{S}_\Sigma$ to which $A \I B$
  is $\mathbf{S}$-provable by $\Sigma$} by
$|A \I B|^\mathbf{S}_\Sigma =
\textstyle
\bigvee\{c \in L;\, \Sigma \vdash^\mathbf{S} A \I c{\otimes}B\}$
for which Theorem~\ref{th:compl} yields that 
$|A \I B|^\mathbf{S}_\Sigma = ||A \I B||^\mathbf{S}_\Sigma$ provided
that $\mathbf{L}$ and $Y$ are finite (otherwise we may
introduce an infinitary cut or its equivalent~\cite{BeVy:Falcrl,KuVy:Flprrai}).
Recall that in the terminology of~\cite[Section~9.2]{Haj:MFL},
this shows that our logic is Pavelka-style~\cite{Pav:Ofl1,Pav:Ofl2,Pav:Ofl3}
complete which means that degrees of semantic entailment (under $\mathbf{S}$)
are exactly the degrees of $\mathbf{S}$-provability. Readers interested in
fuzzy logics admitting this style of completeness are referred
to~\cite{Ger:FL,NoPeMo:MPFL}.

\section{Illustrative Examples}\label{sec:ilex}
In this section, we show examples of concrete parameterizations of FAIs and
show their influence on the number of dependencies derived from object-attribute
data using methods described in Section~\ref{sec:ddd}.
We take the table in Fig.~\ref{fig:data} as the input
data\footnote{\url{http://www.mycoted.com/Comparison_tables}}.

\begin{figure}
  \centering%
  \begin{tabular}[b]{|r||c|c|c|c|}
    \hline
    \rule{0pt}{12pt}&
    happy \underline{k}ids &
    \underline{l}ow cost &
    happy \underline{a}dults &
    \underline{e}asy travel \\
    \hline
    \hline
    walking holiday
    & $0.25$ & $0.75$ & $1$ & $0.75$ \\
    \hline
    cruise holiday
    & $0.5$ & $0.25$ & $0.5$ & $0.5$ \\
    \hline
    beach holiday
    & $0.75$ & $0.25$ & $0.75$ & $0.25$ \\
    \hline
    stay at home
    & $0.25$ & $1$ & $0.5$ & $1$ \\
    \hline
    holiday camp
    & $1$ & $0.25$ & $0.25$ & $0.25$ \\
    \hline
  \end{tabular}
  \caption{Input data:
    Leisure activities and their properties expressed by degrees.}
  \label{fig:data}
\end{figure}

\begin{figure}
  \centering%
  \scalebox{.6}{\includegraphics{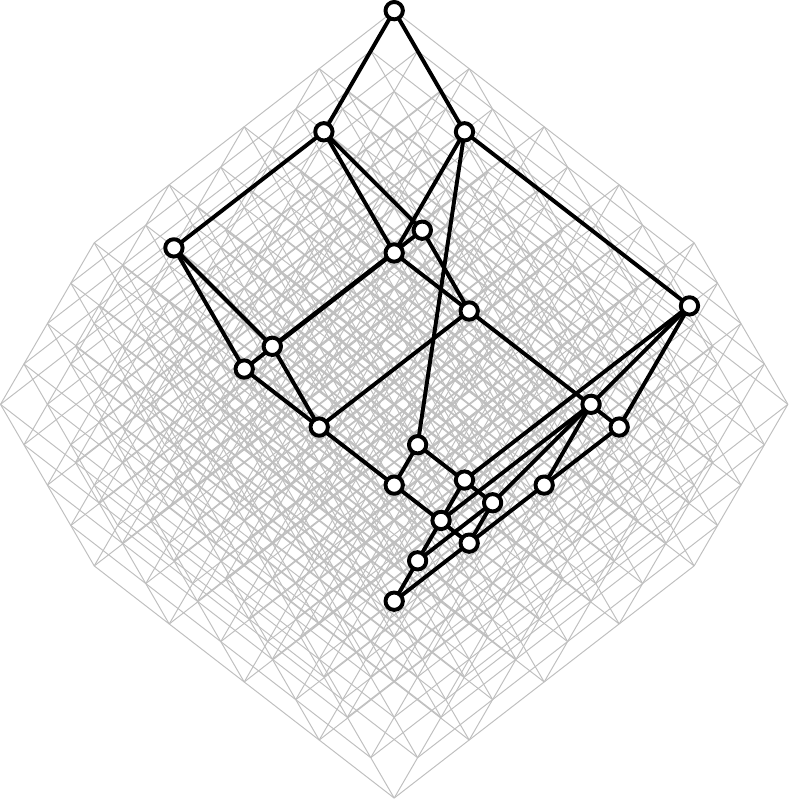}}
  \quad
  \scalebox{.6}{\includegraphics{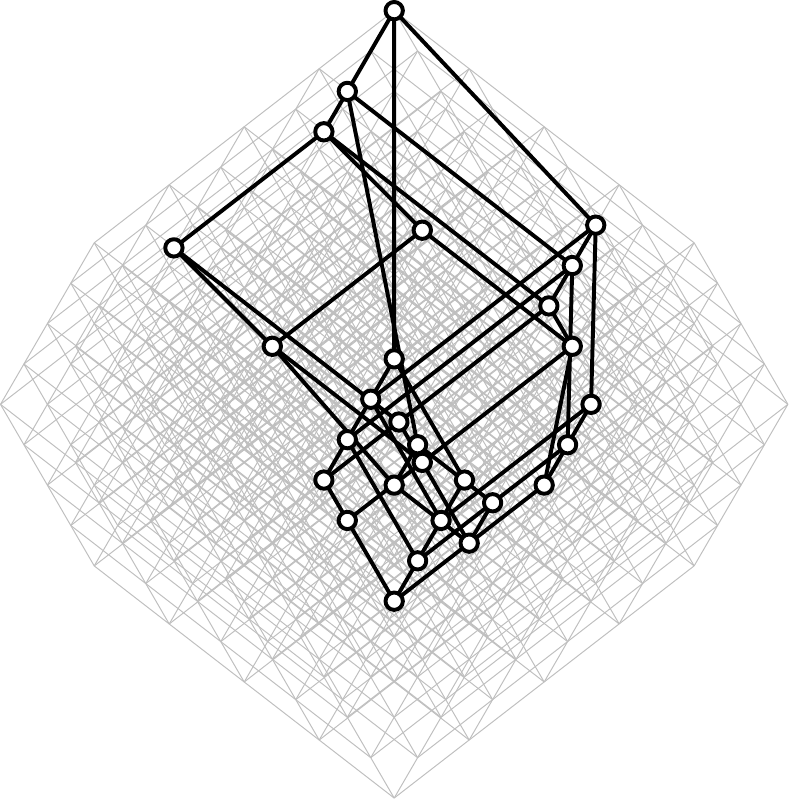}}
  \quad
  \scalebox{.6}{\includegraphics{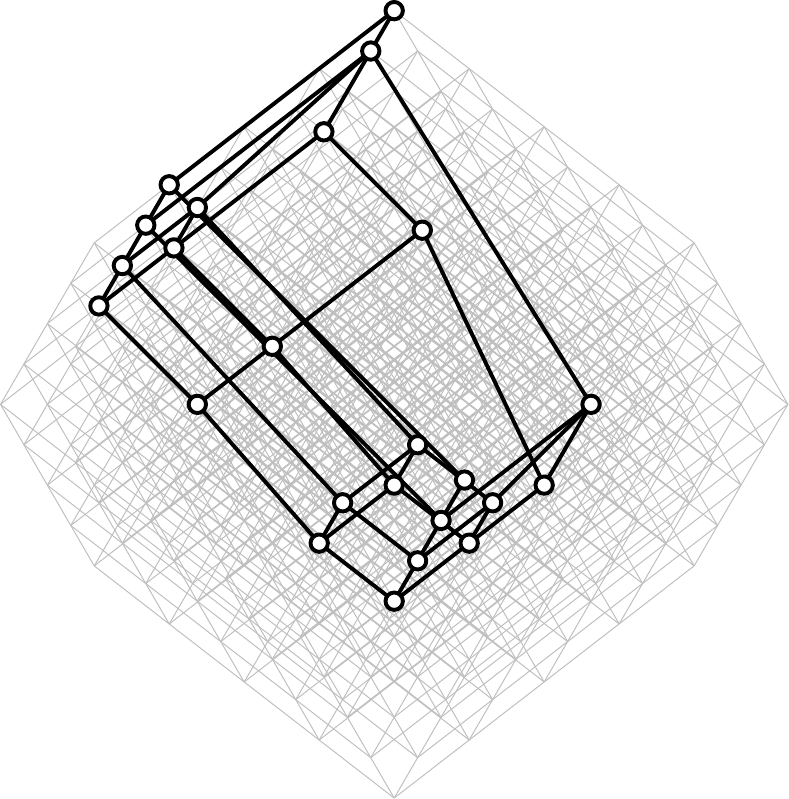}}

  \smallskip
  \scalebox{.6}{\includegraphics{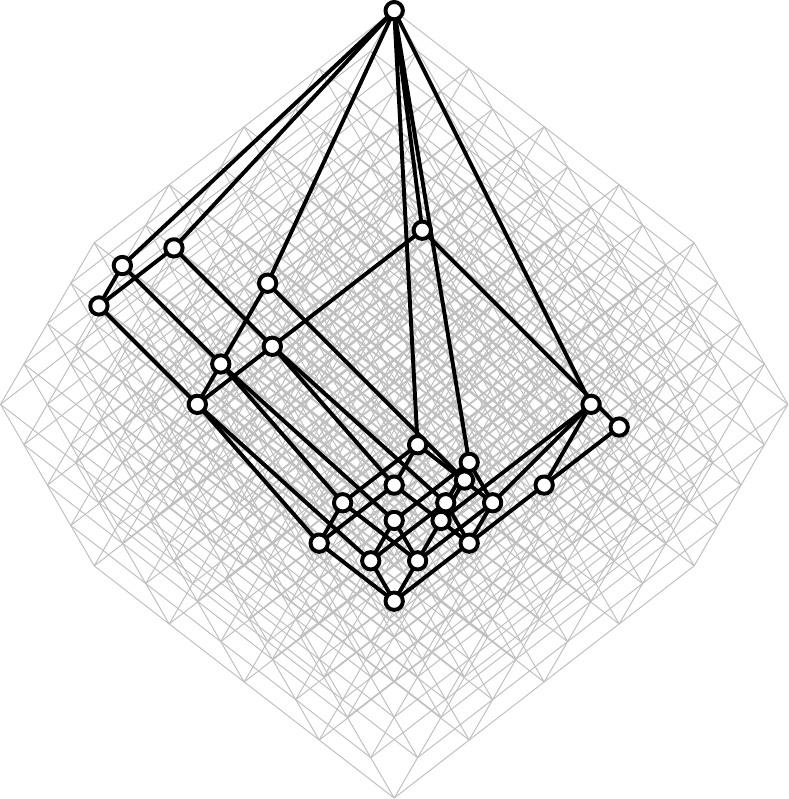}}
  \quad
  \scalebox{.6}{\includegraphics{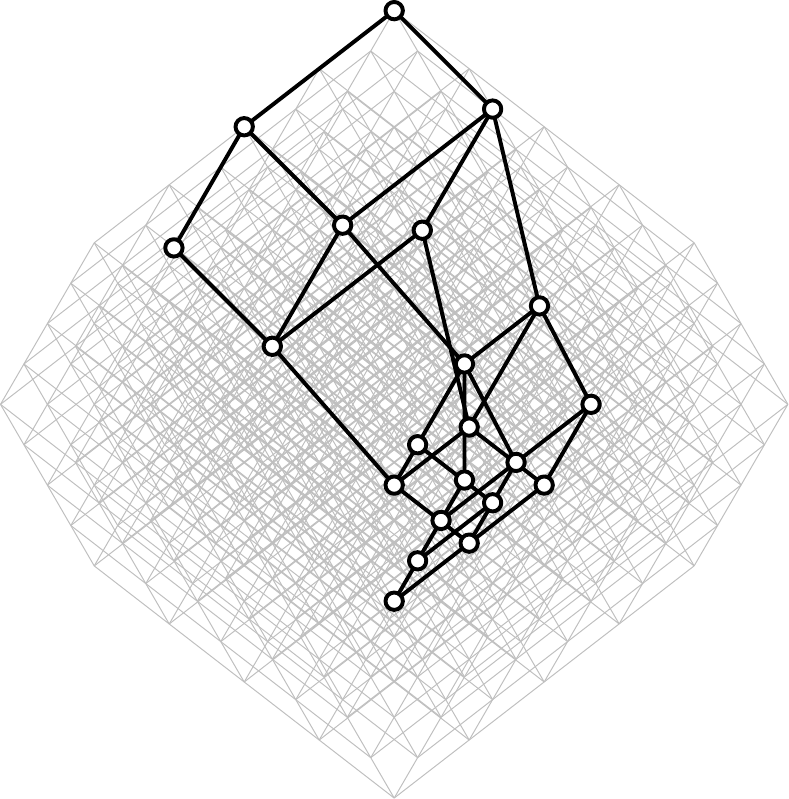}}
  \quad
  \scalebox{.6}{\includegraphics{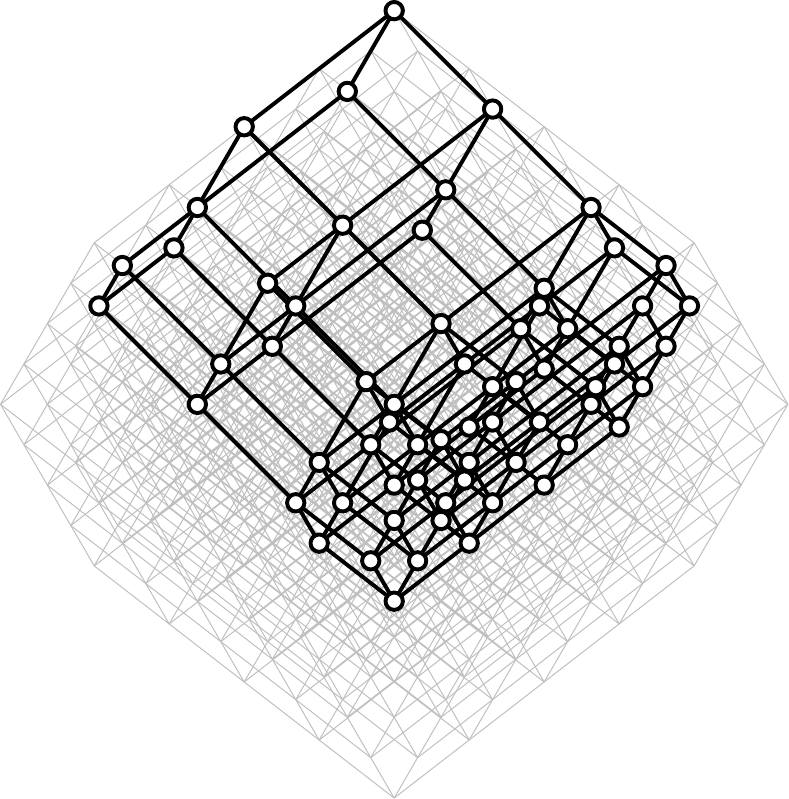}}

  \caption{Systems of fixed points of ${}^{\dn\up}$ from Example~\ref{ex:illu}.}
  \label{fig:space}
\end{figure}

Note that in the terminology of formal concept analysis~\cite{GaWi:FCA}
of data with fuzzy attributes~\cite{Bel:FRS,MeOjRu:Fcavmacl,Po:FB},
the table in Fig.~\ref{fig:data}
represents an $\mathbf{L}$-context with the set $X$ of object being the
types of leisure activities, the set $Y$ of attributes being
properties/features of the activities, and the $\mathbf{L}$-relation $I$
representing the presence of properties by degrees taken from
the real unit interval, e.g.,
\begin{align*}
  I(\text{beach holiday},\text{happy adults}) = 0.75
\end{align*}
means ``beach holiday makes adults happy at least to degree $0.75$.''
In the examples below,
we assume that $\mathbf{L}$ (our structure of degrees) is a complete
residuated lattice on $L = \{0,0.25,0.5,0.75,1\}$ with the G\"odel operations,
i.e., $\wedge$ and $\otimes$ coincide with the minimum, $\vee$ coincides with
the maximum, and $a \rightarrow b = b$ for $a > b$ and $a \rightarrow b = 1$
for $a \leq b$, cf.~\eqref{eqn:LukGodGog_mul} and~\eqref{eqn:LukGodGog_res}.

Given this data, we may be interested in discovering dependencies between
the presence of attributes to be able to answer questions like
``Does a low cost holiday make parents happy (and to what degree)?''
We show by examples that non-redundant bases which describe all such
dependencies present in data as well as their systems of models,
which can be seen as systems of
clusters found in the data~\cite{GaWi:FCA},
are directly influenced by the choice
of a parameterization.

\begin{example}\label{ex:illu}
  (a)
  Consider $S_1 = \{\langle\one,\one\rangle,
  \langle \mul[0.5 \otimes],\shf[0.5 \rightarrow]\rangle\}$ where
  $\mul[0.5 \otimes]$ and $\shf[0.5 \rightarrow]$ are defined by
  \eqref{eqn:MulC} and \eqref{eqn:ShfC}, respectively. This parameterization
  agrees with a parameterization by a hedge with fixed points $0$, $0.5$, and $1$.
  In this case, the set of FAIs given by Theorem~\ref{th:pseudos} which is
  $\mathbf{S}_1$-complete in $\langle X,Y,I\rangle$ is also an $\mathbf{S}_1$-base
  and consists of $11$ formulas:
  \begin{align*}
    \Sigma_I = \{
    &\{k,{}^{0.25\!}/l,{}^{0.75\!}/a,{}^{0.25\!}/e\} \,{\I}\, \{k,{}^{0.25\!}/l,a,{}^{0.25\!}/e\}, \\
    &\{{}^{0.75\!}/k,{}^{0.25\!}/l,a,{}^{0.25\!}/e\} \,{\I}\, \{k,{}^{0.25\!}/l,a,{}^{0.25\!}/e\}, \\
    &\{{}^{0.75\!}/k,{}^{0.25\!}/l,{}^{0.5\!}/a,{}^{0.25\!}/e\} \,{\I}\, \{{}^{0.75\!}/k,{}^{0.25\!}/l,{}^{0.75\!}/a,{}^{0.25\!}/e\}, \\
    &\{{}^{0.5\!}/k,{}^{0.25\!}/l,{}^{0.75\!}/a,{}^{0.25\!}/e\} \,{\I}\, \{{}^{0.75\!}/k,{}^{0.25\!}/l,{}^{0.75\!}/a,{}^{0.25\!}/e\}, \\
    & \{{}^{0.5\!}/k,{}^{0.25\!}/l,{}^{0.5\!}/a,{}^{0.75\!}/e\} \,{\I}\, \{k,{}^{0.25\!}/l,a,e\}, \\
    &\{{}^{0.25\!}/k,l,{}^{0.5\!}/a,{}^{0.75\!}/e\} \,{\I}\, \{{}^{0.25\!}/k,l,{}^{0.5\!}/a,e\}, \\
    &\{{}^{0.25\!}/k,{}^{0.75\!}/l,{}^{0.5\!}/a,e\} \,{\I}\, \{{}^{0.25\!}/k,l,{}^{0.5\!}/a,e\}, \\
    &\{{}^{0.25\!}/k,{}^{0.5\!}/l,{}^{0.25\!}/a,{}^{0.25\!}/e\} \,{\I}\, \{{}^{0.25\!}/k,{}^{0.75\!}/l,{}^{0.5\!}/a,{}^{0.75\!}/e\}, \\
    &\{{}^{0.25\!}/k,{}^{0.25\!}/l,{}^{0.75\!}/a,{}^{0.5\!}/e\} \,{\I}\, \{{}^{0.25\!}/k,{}^{0.25\!}/l,a,{}^{0.75\!}/e\}, \\
    &\{{}^{0.25\!}/k,{}^{0.25\!}/l,{}^{0.25\!}/a,{}^{0.5\!}/e\} \,{\I}\, \{{}^{0.25\!}/k,{}^{0.25\!}/l,{}^{0.5\!}/a,{}^{0.5\!}/e\}, \\
    &0_Y \,{\I}\, \{{}^{0.25\!}/k,{}^{0.25\!}/l,{}^{0.25\!}/a,{}^{0.25\!}/e\}\}.
  \end{align*}
  The $\mathbf{S}_1$-base can be presented in a more compact way by removing
  superfluous attributes from antecedents and consequents of formulas. That is,
  for each $A \I B$ in $\Sigma_I$, we may take minimal $A', B' \in L^Y$ such
  that $A' \subseteq A$ and $B' \subseteq B$ and
  $\Sigma'_I = \{A' \I B'; A \I B \in \Sigma_I\}$ is $\mathbf{S}_1$-complete
  in $\langle X,Y,I\rangle$. For instance,
  \begin{align*}
    \Sigma'_I = \{&\{k,{}^{0.5\!}/a\} \,{\I}\, \{a\},
    \{{}^{0.5\!}/k,a\} \,{\I}\, \{k\},
    \{{}^{0.75\!}/k,{}^{0.5\!}/a\} \,{\I}\, \{{}^{0.75\!}/a\},
    \{{}^{0.5\!}/k,{}^{0.75\!}/a\} \,{\I}\, \{{}^{0.75\!}/k\}, \\
    &\{{}^{0.5\!}/k,{}^{0.75\!}/e\} \,{\I}\, \{{}^{0.75\!}/a,e\},
    \{l\} \,{\I}\, \{e\},
    \{{}^{0.5\!}/l,e\} \,{\I}\, \{l\},
    \{{}^{0.5\!}/l\} \,{\I}\, \{{}^{0.75\!}/l,{}^{0.75\!}/e\}, \\
    &\{{}^{0.75\!}/a,{}^{0.5\!}/e\} \,{\I}\, \{a,{}^{0.75\!}/e\},
    \{{}^{0.5\!}/e\} \,{\I}\, \{{}^{0.5\!}/a\},
    0_Y \,{\I}\, \{{}^{0.25\!}/k,{}^{0.25\!}/l,{}^{0.25\!}/a,{}^{0.25\!}/e\}\}
  \end{align*}
  is an $\mathbf{S}_1$-base obtained from $\Sigma_I$ this way. Let us note that
  according to Theorem~\ref{th:dcompl}, in order to check that $\Sigma'_I$ is
  an $\mathbf{S}_1$-base, it suffices to check that
  $[{\cdots}]^{\mathbf{S}_1}_{\Sigma'_I}$ and ${}^{\dn\up}$ have the same
  fixed points. This can be done by enumerating the fixed points by
  algorithms as in~\cite{BeDeOuVy:Claffco}. Recall that the fixed points
  are $\mathbf{L}$-sets of attributes and play the role of conceptual
  clusters found in the data, see Remark~\ref{rem:clust}. For this
  particular $\langle X,Y,I\rangle$, $\mathbf{L}$, and $\mathbf{S}_1$,
  there are $22$ distinct fixed points (clusters) of ${}^{\dn\up}$.
  Fig.~\ref{fig:space}\,(upper left) depicts the clusters by
  a Hasse diagram (circled vertices and bold edges) drawn in the space
  $\langle L^Y,\subseteq\rangle$ of all $\mathbf{L}$-sets
  (a hypercube with $|L|^{|Y|} = 5^4 = 625$ nodes drawn in gray).

  (b)
  By taking $S_2 = \{\langle\one,\one\rangle,
  \langle \mul[C \otimes],\shf[C \rightarrow]\rangle\}$
  for $C = \{k,{}^{0.5\!}/a,{}^{0.5\!}/e\}$, we introduce a parameterization
  which can be seen as a refinement of that in (a). Described verbally,
  $I \models^{\mathbf{S}_1} A \I B$ means that, for each activity $x$,
  if the activity has all the properties in $A$ (fully or
  at least to degree $0.5$), then it has all the properties in $B$
  (fully or at least to degree $0.5$). In contrast, $\mathbf{S}_2$ puts
  more emphasis on ``happy kids'' (attribute~$k$) because $C(k) = 1 > 0.5$
  and disregards the cost (attribute~$l$) because $C(l) = 0 < 0.5$. Thus,
  by using such a constraint, a user puts more/less emphasis on certain
  attributes. An $\mathbf{S}_2$-base obtained as in (a)
  has $15$ FAIs, and ${}^{\dn\up}$ has $28$ fixed points,
  see Fig.~\ref{fig:space}\,(upper middle).

  (c)
  Considering $S_3$ as in (b) for $C = \{k,{}^{0.75\!}/a,{}^{0.25\!}/e\}$,
  we put more emphasis on ``happy adults'' and less emphasis on ``easy travel''
  than in case of (b). In this setting, an $\mathbf{S}_3$-base consists of $12$
  FAIs and ${}^{\dn\up}$ has $24$ fixed points,
  see Fig.~\ref{fig:space}\,(upper right).

  (d)
  Parameterizations with completely different semantics than in (a)--(c)
  result by considering permutations of attributes. For example, take
  $S_4 = \{\langle\one,\one\rangle,\langle\mul[\rotat],\shf[\rotat]\rangle\}$
  where
  $(\mul[\rotat](A))(k) = A(a)$,
  $(\mul[\rotat](A))(l) = A(e)$,
  $(\mul[\rotat](A))(a) = A(k)$,
  $(\mul[\rotat](A))(e) = A(l)$
  for all $A \in L^Y$,
  and $\shf[\rotat] = \mul[\rotat]$ (notice that $\mul[\rotat]$ is an involution).
  Clearly, $\langle\mul[\rotat],\shf[\rotat]\rangle$ coincide with
  \eqref{eqn:mul+} and \eqref{eqn:shf-} provided that
  we renumber the attributes $k$, $l$, $a$, and $e$ as $0$, $1$, $2$,
  and $3$, respectively. Put in words, 
  $I \models^{\mathbf{S}_1} A \I B$ means that, for each activity $x$,
  if the activity has all the properties
  in $A$ (including situations with
  ``happy adults'' and ``happy kids'' interchanged and
  ``low cost'' and ``easy travel'' interchanged),
  then it has all the properties
  in $B$ (on the same condition of attibutes being interchanged).
  In this case, the $\mathbf{S}_4$-complete set given by
  Theorem~\ref{th:pseudos}
  consists of $17$ implications but unlike the previous cases, the set is
  redundant. Indeed, $7$ formulas can be removed using
  Theorem~\ref{th:base}\,\itm{2} and Theorem~\ref{th:sement}\,\itm{4};
  ${}^{\dn\up}$ has $26$ fixed points, see Fig.~\ref{fig:space}\,(upper left).

  (e)
  By taking $S_5 = \{\langle\one,\one\rangle,
  \langle\mul[\ominus C],\shf[C \oplus]\}$
  for $C = \{k,{}^{0.5\!}/a,{}^{0.5\!}/e\}$ with 
  $\mul[\ominus C]$ and $\shf[C \oplus]$ defined by
  \eqref{eqn:mul_ominus} and \eqref{eqn:shf_oplus}, respectively,
  we introduce a parameterization $\mathbf{S}_5$ which is
  conceptually similar to~$\mathbf{S}_2$ but has a different meaning.
  Indeed, observe that the condition $C \otimes A \subseteq M$ which appears
  in the definition of $I \models^{\mathbf{S}_2} A \I B$ reads:
  For each activity $x$ and each property $y$, $x$ has the property $y$
  \emph{at least to the degree} to which $y$ is in $A$ and 
  $y$ \emph{is prescribed} by $C$. Analogously for $C \otimes B \subseteq M$.
  In case of $I \models^{\mathbf{S}_5} A \I B$, condition 
  $A \ominus C \subseteq M$ reads:
  For each activity $x$ and each property $y$, $x$ has the property $y$
  \emph{at most to the degree} to which $y$ is in $A$ and
  $y$ \emph{is not prescribed} by $C$. It can be shown that if $\mathbf{L}$
  is the \L ukasiewicz structure of degrees, both types of parameterizations
  are mutually reducible but not for general $\mathbf{L}$.
  An $\mathbf{S}_5$-base determined by Theorem~\ref{th:pseudos}
  has $13$ FAIs and ${}^{\dn\up}$ has $21$ fixed points,
  see Fig.~\ref{fig:space}\,(lower middle).

  (f)
  Finally, we consider $\mathbf{S}_6$ which is generated by $S_4 \cup S_5$,
  i.e., $\mathbf{S}_6$ is a parameterization which combines the constraints
  on the semantics of FAIs from (d) and (e).
  Note that in this case, the universe of $S_6$ is not the union of
  $S_4$ and $S_5$ because the union is not closed under compositions. One may
  check that $|S_6| = 8$. An $\mathbf{S}_6$-complete set given by
  Theorem~\ref{th:pseudos} can be reduced to an $\mathbf{S}_6$-base
  consisting of the following formulas (with superfluous attributes removed):
  \begin{align*}
    \Sigma_I = \{&0_Y \,{\I}\, \{{}^{0.25\!}/a,{}^{0.25\!}/e\},
    \{{}^{0.75\!}/l\} \,{\I}\, \{{}^{0.75\!}/e\},
    \{l,{}^{0.75\!}/a\} \,{\I}\, \{{}^{0.5\!}/k,e\},
    \{{}^{0.75\!}/k,{}^{0.5\!}/e\} \,{\I}\, \{k\}\}.
  \end{align*}
  In this case, ${}^{\dn\up}$ has $65$ fixed points,
  cf. Fig.~\ref{fig:space}\,(lower right). We conclude the examples by showing
  that $\{{}^{0.75\!}/a,e\} \I \{{}^{0.5\!}/k,l,a\}$ is semantically entailed
  by $\Sigma_I$ under $\mathbf{S}_6$. According to Theorem~\ref{th:compl},
  it suffices to show that $\Sigma_I \proves^{\mathbf{S}_6}
  \{{}^{0.75\!}/a,e\} \I \{{}^{0.5\!}/k,l,a\}$ which is indeed the case:
  \def\ProofFmt#1#2{\makebox[14em][l]{#1}~#2}
  \begin{enumerate}\parskip=-5pt
  \item
    \ProofFmt{$\{l,{}^{0.75\!}/a\} \,{\I}\, \{{}^{0.5\!}/k,e\}$}{%
      formula in $\Sigma_I$,}
  \item
    \ProofFmt{${\{e,{}^{0.75\!}/k\} \,{\I}\, \{l,{}^{0.5\!}/a\}}$}{%
      \eqref{r:F} for $\mul[\rotat]$ applied to 1,}
  \item
    \ProofFmt{$\{e\} \,{\I}\, \{l\}$}{%
      \eqref{r:F} for $\mul[\ominus C]$
      with $C = \{k,{}^{0.5\!}/a,{}^{0.5\!}/e\}$ applied to 2,}
  \item
    \ProofFmt{$\{{}^{0.75\!}/a,e\} \,{\I}\, \{l,{}^{0.75\!}/a,e\}$}{%
      \eqref{r:Cut} applied to 3 and axiom
      $\{l,{}^{0.75\!}/a,e\} \,{\I}\, \{l,{}^{0.75\!}/a,e\}$,}
  \item
    \ProofFmt{$\{l,{}^{0.75\!}/a\} \,{\I}\, \{{}^{0.5\!}/k,e\}$}{%
      formula in $\Sigma_I$,}
  \item
    \ProofFmt{$\{l,{}^{0.75\!}/a,e\} \,{\I}\, \{{}^{0.5\!}/k,l,{}^{0.75\!}/a,e\}$}{%
      \eqref{r:Cut} applied to 5 and $\{{}^{0.5\!}/k,l,{}^{0.75\!}/a,e\} \,{\I}\, \{{}^{0.5\!}/k,l,{}^{0.75\!}/a,e\}$,}
  \item
    \ProofFmt{$\{{}^{0.75\!}/a,e\} \,{\I}\, \{{}^{0.5\!}/k,l,{}^{0.75\!}/a,e\}$}{%
      \eqref{r:Cut} applied to 4 and 6,}
  \item
    \ProofFmt{$\{{}^{0.75\!}/k,{}^{0.5\!}/e\} \,{\I}\, \{k\}$}{%
      formula in $\Sigma_I$,}
  \item
    \ProofFmt{$\{{}^{0.5\!}/l,{}^{0.75\!}/a\} \,{\I}\, \{a\}$}{%
      \eqref{r:F} for $\mul[\rotat]$ applied to 8,}
  \item
    \ProofFmt{$\{{}^{0.5\!}/k,l,{}^{0.75\!}/a,e\} \,{\I}\, \{{}^{0.5\!}/k,l,a,e\}$}{%
      \eqref{r:Cut} applied to 9 and $\{{}^{0.5\!}/k,l,a,e\} \,{\I}\, \{{}^{0.5\!}/k,l,a,e\}$,}
  \item
    \ProofFmt{$\{{}^{0.75\!}/a,e\} \,{\I}\, \{{}^{0.5\!}/k,l,a,e\}$}{%
      \eqref{r:Cut} applied to 7 and 10,}
  \item
    \ProofFmt{$\{{}^{0.75\!}/a,e\} \,{\I}\, \{{}^{0.5\!}/k,l,a\}$}{%
      \eqref{r:Cut} applied to 11 and $\{{}^{0.5\!}/k,l,a,e\} \I \{{}^{0.5\!}/k,l,a\}$.}
  \end{enumerate}
\end{example}

\section{Conclusion}\label{sec:concl}
General family of if-then rules parameterized by systems of \monoton e Galois
connections has been investigated. Bivalent and graded notions of semantic
entailment of if-then rules have been characterized in terms of least models
and complete axiomatization has been provided.
Non-redundant bases of if-then rules derived from object-attribute data with
fuzzy attributes have been characterized using operators on fuzzy sets induced
by data. Several examples of parameterizations have been shown. Future research
will focus on applications of the parameterizations in formal concept analysis,
metods for data dimensionality reduction, and related areas where the earlier
parameterizations by hedges have been successfully applied.

\subsection*{Acknowledgment}
Supported by grant no. \verb|P202/14-11585S| of the Czech Science Foundation.


\footnotesize\openup=-4pt
\bibliographystyle{amsplain}
\bibliography{pasefai}

\end{document}